\documentclass[twoside]{article}

\usepackage[accepted]{aistats2019}
\usepackage[round]{natbib}

\usepackage[parfill]{parskip} 
\usepackage{graphicx}
\usepackage{epstopdf}
\usepackage[export]{adjustbox}
\usepackage{caption}
\usepackage[font={small}]{caption} 
\usepackage[round]{natbib}
\usepackage{appendix}
\usepackage{algorithm}
\usepackage{algcompatible}
\usepackage{amssymb,amsmath,amsthm}

\usepackage{hyperref}
\usepackage{color}
\hypersetup{
	colorlinks=true, 
	linkcolor=blue,  
	citecolor=blue,  
	filecolor=blue,  
	urlcolor=blue,   
}

\newcommand{\R}{\mathbb{R}}

\newcommand{\rbra}[1]{\left(#1\right)}
\newcommand{\sbra}[1]{\left[#1\right]}
\newcommand{\cbra}[1]{\left\{#1\right\}}

\newcommand{\norm}[1]{\left\|#1\right\|}
\newcommand{\inner}[1]{\langle#1\rangle}
\newcommand{\abs}[1]{\left|#1\right|}

\newcommand{\wh}[1]{\widehat{#1}}
\newcommand{\wb}[1]{\overline{#1}}

\newcommand{\var}{\mathrm{var}}
\newcommand{\cov}{\mathrm{cov}} 
\newcommand{\cum}{\mathrm{cum}} 
\newcommand{\E}{\mathbb{E}}

\newcommand{\ccal}{\mathcal{C}}
\newcommand{\ecal}{\mathcal{E}}
\newcommand{\lcal}{\mathcal{L}}
\newcommand{\kcal}{\mathcal{K}}
\newcommand{\ncal}{\mathcal{N}}
\newcommand{\pcal}{\mathcal{P}}
\newcommand{\scal}{\mathcal{S}}

\newcommand{\vect}{\mathrm{vec}}
\newcommand{\mat}{\mathrm{mat}}
\newcommand{\spn}{\mathrm{Span}}
\newcommand{\tr}{\mathrm{Tr}}
\newcommand{\diag}{\mathrm{Diag}}
\newcommand{\proj}{\mathrm{Proj}}

\newcommand{\acos}{\mathrm{acos}}

\def\min{\mathop{\rm min}\limits}
\def\argmin{\mathop{\rm argmin}\limits}

\def\maximize{\mathop{\rm maximize}\limits}
\def\max{\mathop{\rm max}\limits}
\def\argmax{\mathop{\rm argmax}\limits}

\makeatletter
\newtheorem*{rep@theorem}{\rep@title}
\newcommand{\newreptheorem}[2]{%
	\newenvironment{rep#1}[1]{%
		\def\rep@title{#2 \ref{##1}}%
		\begin{rep@theorem}}%
		{\end{rep@theorem}}}
\makeatother
\newtheorem{lemma}{Lemma}[subsection] 
\newtheorem{proposition}{Proposition}[section]
\newtheorem{theorem}{Theorem}[section]
\newtheorem*{theorem-nonumber}{Theorem}

\newtheorem{assumption}{Assumption}[section]
\newtheorem{program}[theorem]{Program}
\newreptheorem{program}{Program}

\newcommand{\defeq}{\triangleq}

\newcommand{\one}{\mathbbm{1}}

\DeclareMathOperator{\probP}{\Pr}
\newcommand{\prob}[1]{\probP[ #1 ]}
\newcommand{\problr}[1]{\probP\left[ #1 \right]}
\newcommand{\EE}{\mathbb{E}}  
\DeclareMathOperator{\Binom}{Binom}

\newcommand{\iid}{i.i.d.\ }

\usepackage{dsfont} 
\usepackage{bbm}

\begin{document}

%

%

\twocolumn[

\aistatstitle{Overcomplete Independent Component Analysis via SDP}

\runningauthor{Podosinnikova, Perry, Wein, Bach, d'Aspremont, Sontag}

\aistatsauthor{ Anastasia Podosinnikova \And Amelia Perry \And  Alexander S.\ Wein}

\aistatsaddress{ MIT \And MIT \And Courant Institute, NYU} 

\aistatsauthor{Francis Bach \And Alexandre d'Aspremont \And David Sontag }

\aistatsaddress{INRIA, ENS \And  CNRS, ENS \And MIT }

]

\begin{abstract}
  We present a novel algorithm for overcomplete independent components analysis (ICA), 
  where the number of latent sources $k$ exceeds the dimension $p$ of observed variables. 
  Previous algorithms either suffer from high computational complexity or make strong 
  assumptions about the form of the mixing matrix.
  Our algorithm does not make any sparsity assumption
  yet enjoys favorable computational and theoretical properties.
  Our algorithm consists of two main
  steps: (a) estimation of the Hessians of the cumulant generating function
  (as opposed to the fourth and higher order cumulants used by most algorithms) and (b) 
  a novel semi-definite programming (SDP) relaxation for recovering a 
  mixing component. We show that this relaxation can be 
  efficiently solved with a projected accelerated gradient descent method,
  which makes the whole algorithm computationally practical. 
  Moreover, we conjecture that the proposed program recovers a mixing component at the rate $k < p^2/4$ and prove that a mixing component can be recovered with high probability when 
  $k < (2 - \varepsilon)p \log p$ when the original components are sampled uniformly at random on the hypersphere. 
  Experiments are provided on synthetic data and the
  CIFAR-10 dataset of real images.
\end{abstract}

\section{Introduction}

	\emph{Independent component analysis (ICA)}
	models a $p$-dimensional \emph{observation}
	$x$ as a linear combination 
	of $k$ latent mutually
	independent \emph{sources}:
	\begin{equation}
	\label{ica}
	x = D\alpha,
	\end{equation}
	where $\alpha := (\alpha_1, \dots, \alpha_k)^{\top}$ and $D \in \R^{p\times k}$.
	The linear transformation $D$ is called the \emph{mixing matrix}
	and is closely related to the \emph{dictionary matrix}
	from dictionary learning \citep[see, e.g.,][]{CheDon1994,CheEtAl1998}. 
	Given a sample $X := \{x^{(1)}, \dots, x^{(n)}\}$ of $n$ observations,
	one is often interested in estimating the latent mixing matrix $D$ and 
	respective latent representations, $\alpha^{(1)}, \dots, \alpha^{(n)}$, also known as \emph{sources}, 
	of every observation.
	
	A classical motivating example for  ICA  
	is the cocktail party problem, where one is interested in separating 
	individual speakers' voices from noisy recordings. Here, 
	each record is an observation and each speaker is an independent source.
	In general, ICA is a simple single-layered neural network and 
	is widely used as an unsupervised learning method 
	in machine learning and
	signal processing communities \citep[see, e.g.,][]{HyvEtAl2001,ComJut2010}.
	
	There are three conceptually different settings 
	of the ICA problem: (a) \emph{complete},
	or \emph{determined}, where the dimension of observations coincides
	with the number of sources, i.e., $p=k$; (b) \emph{undercomplete},
	or \emph{overdetermined}, with fewer sources than the dimension, i.e., $k<p$;
	and (c) \emph{overcomplete}, or \emph{underdetermined}, with more sources than
	the dimension, i.e., $k>p$. While the first two cases are well studied, 
	the last one is more difficult and we address it here. 
	
	In the \emph{complete} setting, where $k=p$, ICA is usually
	solved via pre-whitening of the data so that
	the whitened observations, $z := Wx$, are uncorrelated and
	all have unit variance, i.e., $\cov(z) = W\cov(x)W^{\top} = I$,
	where $W$ denotes the whitening matrix.
	Substituting $x = D\alpha$, we get 
	$(WD) (WD)^{\top} = I$ which implies that the matrix $Q := WD$
	is \emph{orthogonal} and therefore the problem of finding the 
	mixing matrix $D$ boils down to finding the ``correct'' orthogonal 
	matrix $Q$. Numerous ``correctness'' criteria, such as maximizing
	non-Gaussianity of sources, were proposed and
	respective algorithms for complete ICA
	are well known
	\citep[see, e.g.,][]{HyvEtAl2001,ComJut2010}. 
	The most widely known complete ICA algorithms are
	possibly the \emph{FastICA} algorithm by
	\citet{Hyv1999} and 
	the \emph{JADE} algorithm by
	\citet{CarSou1993}.
	This 
	naturally extends to the undercomplete setting where one looks 
	for an orthonormal matrix, where columns are orthogonal, instead.
	However, although nothing prevents us from whitening data 
	in the overcomplete setting, the orthogonalization trick 
	cannot be extended to the 
	\emph{overcomplete} setting, where $k>p$, since the mixing matrix
	$D$ has more columns than rows and therefore cannot 
	have full column rank. 
	
	Improvements in feature learning are among the advantages
	of overcomplete representations: it has been shown by \citet{CoaEtAl2011}
	that dense and overcomplete features
	can significantly improve performance of classification algorithms.
	However, advantages of overcomplete representations go far
	beyond this task \citep[see, e.g.,][]{BenEtAl2013}.
	
	Originally, the idea of overcomplete representations 
	was developed in the context of dictionary learning,
	where an overcomplete dictionary, formed by Fourier, wavelet,
	Gabor or other filters, is given and one is only interested in 
	estimating the latent representations~$\alpha$.
	Different approaches were proposed for this problem
	including the method of frames \citep{Dau1988} and 
	basis pursuit \citep{CheDon1994,CheEtAl1998}.
	Later in sparse coding, the idea of estimating a dictionary matrix
	directly from data was introduced \citep{OlsFie1996,OlsFie1997}
	and was shortly followed by the first overcomplete ICA algorithm
	\citep{LewEtAl2000}.\footnote{
		Recall the close relation between ICA and sparse coding:
		indeed, the maximum likelihood estimation of ICA with the Laplace prior 
		on the sources (latent representations $\alpha$) is equivalent to the
		standard sparse coding formulation with the $\ell_1$-penalty.
	}
	Further overcomplete ICA research continued in several fairly 
	different directions based on either (a) 
	various sparsity assumptions \citep[see, e.g.,][]{TehEtAl2003}
	or on (b) prior assumptions about the sources as by \citet{LewEtAl2000}
	or (c) instead in a more general dense overcomplete setting 
	\citep[see, e.g.,][]{Hyv2005,ComRaj2006,LatEtAl2007,GoyEtAl2014,BhaEtAl2014,BhaEtAl2014b,AnaEtAl2015,MaEtAl2016}. Since we focus
	on this more general dense setting, we do not review
	or compare to the literature in the other settings.
	
	In particular, we focus on the following problem: \emph{Estimate the mixing matrix
		$D$ given an observed sample $X := \cbra{x^{(1)}, \dots, x^{(n)}}$ of $n$ observations}.
	We aim at constructing an algorithm that would bridge the 
	gap between algorithms with theoretical guarantees and
	ones with practical computational properties. 
	Notably, our algorithm does not depend on 
	any probabilistic assumptions on the sources, except for the 
	standard independence and non-Gaussianity,
	and 
	the uniqueness of the ICA representation (up to permutation and scaling) is 
	the result of the independence of sources rather than sparsity.
	Here we only focus on the estimation of the latent mixing matrix
	and leave the learning of the latent representation for future 
	research
	(note that one can use, e.g., the mentioned earlier
	dictionary learning approaches).
	
	Different approaches have been proposed to address
	this problem. Some attempt to relax the
	hard orthogonality constraint in  
	the whitening procedure with more heuristic
	quasi-orthogonalization approaches \citep[see, e.g.,][]{LeEtAl2011,AroEtAl2012}.
	Other approaches try to specifically address the structure of the 
	model in the overcomplete setting 
	\citep[see, e.g.,][]{Hyv2005,ComRaj2006,LatEtAl2007,GoyEtAl2014,BhaEtAl2014,BhaEtAl2014b,AnaEtAl2015,MaEtAl2016}
	by considering higher-order cumulants or derivatives of the
	cumulant generating function.
	The algorithm that we propose is the closest to the
	latter type of approach.
	
	\begin{algorithm}[b]
		\caption{OverICA}
		\label{alg-overica}
		\begin{algorithmic}[1]
			\STATE \textbf{Input: Observations $X:=\cbra{x_1, \dots, x_n}$ and
				latent dimension $k$.}
			\STATEx Parameters: The regularization parameter $\mu$ and 
			the number $s$ of generalized covariances, $s>k$.
			\STATE \textbf{STEP I. Estimation of the subspace $W$:}
			\STATEx Sample vectors $t_1, \dots, t_s$.
			\STATEx Estimate matrices $H_j := \ccal_x(t_j)$ for all $j\in[s]$.
			\STATE \textbf{STEP II. Estimation of the atoms:}
			\STATEx Given $G^{(i)}$ for every deflation step $i = 1,2,\dots, k$:
			\STATEx Solve the relaxation~\eqref{def-sdp-relax} with $G^{(i)}$.
			\STATEx (OR: Solve the program~\eqref{def-sdp} with $G^{(i)}$.)
			\STATEx Estimate the $i$-th mixing component $d_i$ from $B^{\ast}$.
			\STATE \textbf{Output: Mixing matrix $D = (d_1, d_2, \dots, d_k)$.}
		\end{algorithmic}
	\end{algorithm}
	
	We make two conceptual contributions: (a) we show how to use second-order statistics instead
	of the fourth and higher-order cumulants, which improves sample complexity, and (b) we introduce a novel semi-definite
	programming-based approach, with a convex relaxation that can be solved efficiently, for estimating the columns of $D$. 
	Overall, this leads to a computationally efficient overcomplete ICA algorithm
	that also has theoretical guarantees.
	Conceptually, our work is similar to the fourth-order only 
	blind identification (FOOBI) algorithm \citep{LatEtAl2007},
	which we found to work well in practice.
	However, FOOBI suffers from high computational and memory complexities, 
	its theoretical guarantee requires all kurtoses of the sources to be positive, 
	and it makes the strong assumption that certain fourth-order tensors are linearly independent.
	Our approach resolves these drawbacks.
	We describe our algorithm in Section~\ref{sec-main} and
	 experimental results in Section~\ref{sec-exps}.

\section{Overcomplete ICA via SDP}
\label{sec-main}
	
\subsection{Algorithm overview}
	
	We focus on estimating the latent mixing
	matrix $D\in\R^{p\times k}$
	of the ICA model~\eqref{ica} in the overcomplete
	setting where $k>p$. 
	We first motivate our algorithm in the population 
	(infinite sample) setting and later address 
	the finite sample case.
	
	In the following, the $i$-th column of the mixing matrix $D$ is denoted as~$d_i$
	and called the $i$-th \textbf{mixing component}.
	The rank-1 matrices $d_1d_1^{\top}, \dots, d_kd_k^{\top}$ are referred to 
	as \textbf{atoms}.\footnote{
		We slightly abuse the standard closely related dictionary
		learning terminology where the term atom 
		is used for the individual columns $d_i$
		\citep[see, e.g.,][]{CheEtAl1998}.
	}

	Our algorithm, referred to as \textbf{OverICA}, consists of two major steps:
	(a) construction of the \textbf{subspace~$\mathbf{W}$} spanned 
	by the atoms, i.e.,
	\begin{equation}
		\label{def-w}
		\setlength\abovedisplayskip{.5mm}
		W:= \spn\cbra{d_1d_1^{\top}, \dots,d_kd_k^{\top}},
		\vspace{-2mm}
	\end{equation}
	and $(b)$ estimation of individual atoms  $d_id_i^{\top}$, $i\in[k]$,
	given any basis of this subspace.\footnote{
		The mixing component is then the largest eigenvector. 
	}
	We summarize this high level idea\footnote{
		The deflation part is more involved (see Section~\ref{sec-deflation}).
	}
	in Algorithm~\ref{alg-overica}.
	Note that although the definition of 
	the subspace
	$W$ in~\eqref{def-w} is based on the latent atoms, in practice
	this subspace is estimated from the known observations~$x$
	(see Section~\ref{sec-subspace}).
	However, we do use this explicit representation in our theoretical analysis.
	
	In general, there are different ways to implement
	these two steps.
	For instance, some algorithms implement the first step based on the fourth or higher order
	cumulants \citep[see, e.g.,][]{LatEtAl2007,GoyEtAl2014}. 
	In contrast, we estimate the subspace $W$ from the 
	Hessian of the cumulant generating function
	which has better computational and sample complexities (see Section~\ref{sec-subspace}).
	Our algorithm also works (without any adjustment) with other implementations
	of the first step, including the fourth-order cumulant based one,
	but other algorithms cannot take advantage of our
	efficient first step due to the differences in the second step.

	In the second step, 
	we propose a novel 
	semi-definite program (SDP) for estimation 
	of an individual atom given the subspace $W$ (Section~\ref{sec-sdp}).
	We also provide a convex relaxation of this program which 
	admits efficient implementation and introduces regularization to
	noise which is handy in practice when the subspace $W$ can only
	be estimated approximately (Section~\ref{sec-relaxation}).
	Finally, we provide a deflation procedure that allows us to estimate
	all the atoms (Section~\ref{sec-deflation}).
	Before proceeding, a few assumptions are in order.

\subsection{Assumptions}
	
	Due to the inherent permutation and scaling unidentifiability of the ICA problem, 
	it is a standard practice to assume, without loss of generality, that 
	\begin{assumption}
		\label{ass-scaling}
		Every mixing component has unit norm, i.e.,
		$\norm{d_i}_2 = 1$ for all $i\in[k]$. 
	\end{assumption}
	This assumption immediately implies that
	all atoms have unit Frobenius norm, i.e.,
	$\norm{d_id_i^{\top}}_F  = \norm{d_i}_2^2 = 1$ for all $i\in[k]$.
	
	Since instead of recovering mixing components $d_i$ as in (under-) complete setting
	we recover atoms $d_id_i^{\top}$, the following assumption is necessary for the identifiability of our algorithm:
	\begin{assumption}
		\label{ass-lnd}
		The matrices (atoms) $d_1d_1^{\top}$, $d_2d_2^{\top}$, \dots, $d_kd_k^{\top}$
		are linearly independent.
	\end{assumption}
	This in particular implies that the number of sources~$k$
	cannot exceed $m:=p(p+1)/2$, which is the 
	latent dimension of the set of all symmetric matrices $\scal_p$.
	We also assume, without loss of generality, that
	the observations are centred, i.e., $\E(x) = \E(\alpha) = 0$.

\subsection{Step I: Subspace Estimation}
\label{sec-subspace}
	
	In this section, we describe a construction of an orthonormal basis of
	the subspace $W$.
	For that, we first construct matrices $H_1, \dots, H_s \in\R^{p\times p}$,
	for some~$s$, which span the subspace $W$. These matrices
	are obtained from the Hessian of the cumulant generating function
	as described below.

	\paragraph{Generalized Covariance Matrices.}
	Introduced for complete ICA  by~\citet{Yer2000},
	a generalized covariance matrix is the Hessian of the cumulant generating
	function evaluated at a non-zero vector. 
	
	Recall that the cumulant generating function (cfg) of a $p$-valued
	random variable $x$ is defined as
	\begin{equation}
		\label{cgf}
		\setlength\abovedisplayskip{.5mm}
		\phi_x(t) := \log \E (e^{t^{\top} x}),
		\vspace{-2.3mm}
	\end{equation}
	for any $t\in\R^p$. It is well known that the cumulants of $x$ 
	can  be computed as the coefficients of the Taylor series expansion of the cgf
	evaluated at zero~\citep[see, e.g.,][Chapter 5]{ComJut2010}. In particular, the second order cumulant,
	which coincides with the covariance matrix, is then the Hessian evaluated at zero, i.e., $\cov(x) = \nabla^2 \phi_x(0)$.
		
	The \textbf{generalized covariance matrix}
	is a straightforward extension
	where the Hessian of the cgf is
	evaluated at a non-zero vector $t$:
	\begin{equation}
	\label{gencov}
		\setlength\abovedisplayskip{0mm}
		\ccal_x(t) := \nabla^2 \phi_x(t) 
		= \frac{ \E(xx^{\top} e^{t^{\top}x}) }{ \E(e^{t^{\top}x}) } 
		- \ecal_x(t) \ecal_x(t)^{\top},
		\vspace{-3mm}
	\end{equation}
	where we introduced
	\begin{equation}
		\label{genmean}
		\setlength\abovedisplayskip{0mm}
		\ecal_x(t) := \nabla \phi_x(t) = \frac{ \E(xe^{t^{\top}x}) }{ \E(e^{t^{\top}x}) }.
	\end{equation}

\paragraph{Generalized Covariance Matrices of ICA.}
	
	In case of the ICA model, 
	substituting~\eqref{ica} into the expressions~\eqref{genmean} and~\eqref{gencov}, we obtain
	\begin{equation}
		\label{ica-gencov}
		\setlength\abovedisplayskip{.5mm}
		\begin{aligned}
			\ecal_x(t) &= \frac{ D \E(\alpha e^{\alpha^{\top}y}) }{ \E( e^{\alpha^{\top}y}) } = D \ecal_{\alpha}(y), \\
			\ccal_x(t) &= D \ccal_{\alpha}(y) D^{\top},
		\end{aligned}
		\vspace{-1mm}
	\end{equation}
	where we introduced $y := D^{\top}t$ and
	the generalized covariance  
	$\ccal_{\alpha}(y) := \nabla^2 \phi_{\alpha}(y)$ of the sources:
	\begin{equation}
		\label{gencov-alpha}
		\setlength\abovedisplayskip{.5mm}
		\ccal_{\alpha}(y) = \frac{ \E(\alpha\alpha^{\top} e^{\alpha^{\top}y}) }{ \E(e^{\alpha^{\top}y}) } 
		- \ecal_{\alpha}(y) \ecal_{\alpha}(y)^{\top},
		\vspace{-1mm}
	\end{equation}
	where $\ecal_{\alpha}(y) := \nabla \phi_{\alpha}(y)=\E(\alpha e^{y^{\top}\alpha})/
	\E(e^{y^{\top}\alpha})$.
		
	Importantly, the generalized covariance $\ccal_{\alpha}(y)$ of the sources, 
	due to the independence,
	is a diagonal matrix~\citep[see, e.g.,][]{PodEtAl2016}. Therefore, 
	the ICA generalized covariance  $\ccal_x(t)$ is:
	\begin{equation}
		\label{gencov-sum}
		\setlength\abovedisplayskip{-.5mm}
		C_x(t) = \sum_{i=1}^k \omega_i(t) d_i d_i^{\top},
		\vspace{-2mm}
	\end{equation}
	where $\omega_i(t) := [\ccal_{\alpha}(D^{\top}t)]_{ii}$ are the generalized variance of the $i$-th source $\alpha_i$. 
	This implies that \emph{ICA generalized
		covariances belong to the subspace~$W$.}

	\paragraph{Construction of the Subspace.}
	Since ICA generalized covariance matrices 
	belong to the subspace $W$, then the span of any number of such 
	matrices would either be a subset of $W$ or equal to $W$.
	Choosing sufficiently large number $s>k$ of generalized covariance matrices,
	we can ensure the equality.
	Therefore, given a sufficiently large number $s$ of vectors $t_1, \dots, t_s$, 
	we construct matrices
	$H_j := \ccal_x(t_j)$ for all $j\in[s]$. 
	Note that in practice it is more convenient to
	work with vectorizations of these matrices
	and then consequent matricization of the obtained
	result (see Appendices~\ref{app-mat-vec} and~\ref{app-ica-cum4}).
	Given matrices $H_j$, for $j\in[s]$, an orthonormal basis 
	can be straightforwardly extracted via the singular value decomposition.
	In practice, we set $s$ as a multiple of $k$ and 
	sample the vectors $t_j$ from the Gaussian distribution.
	
	Note that one can also construct a basis of the subspace $W$
	from the column space of the flattening of the fourth-order
	cumulant of the ICA model~\eqref{ica}. In particular,
	this flattening is a matrix $C \in\R^{p^2\times p^2}$ such that
	$C = (D \odot D) \diag(\kappa) (D\odot D)$,
	where $\odot$ stands for the Khatri-Rao product and the $i$-th element
	of the vector $\kappa\in\R^k$ is the kurtosis of the $i$-th source~$\alpha_i$.
	Importantly, matricization of the $i$-th column~$a_i$ of the matrix $A := D\odot D$
	is exactly the $i$-th atom, i.e., $\mat(a_i) = d_id_i^{\top}$. Therefore, one
	can construct the desirable basis from the column space of the matrix $A$
	(see Appendix~\ref{app-ica-cum4} for more details).
	This also intuitively explains the need for Assumption~\ref{ass-lnd},
	which basically ensures that $A$ has full column rank (as opposed to $D$). 
	In general, this approach is 
	common in the overcomplete literature \citep[see, e.g.,][]{LatEtAl2007,BhaEtAl2014,AnaEtAl2015,MaEtAl2016}
	and can be used as the first step of our algorithm.
	However, the generalized covariance-based construction has better computational (see Section~\ref{sec-runtime})
	and sample complexities.

\subsection{Step II: Estimation of the Atoms}
	
	We now discuss the recovery  of one atom
	$d_id_i^{\top}$, for some $i\in[k]$, given a basis of the subspace $W$ (Section~\ref{sec-sdp}). 
	We then
	provide a deflation procedure
	to recover all atoms $d_id_i^{\top}$ (Section~\ref{sec-deflation}).

\subsubsection{The Semi-Definite Program}
\label{sec-sdp}
	
	Given matrices $H_1, H_2, \dots, H_s$ which span the subspace $W$ defined in~\eqref{def-w}
	we formulate
	the following \emph{semi-definite program (SDP)}:
	\begin{equation}
		\label{def-sdp}
		\setlength\abovedisplayskip{.1mm}
		\begin{aligned}
			B^{\ast}_{sdp} := &\argmax_{B\in\scal_p} \; \inner{G,B} \\
			& B \in \spn\cbra{H_1, H_2, \dots, H_s}, \\
			& \tr(B) = 1, \\
			& B \succeq 0.
		\end{aligned}
		\vspace{-3mm}
	\end{equation}
	We expect that the optimal solution (if it exists and is unique)
	$B^{\ast}_{sdp}$ coincides with one of the atoms $d_id_i^{\top}$
	for some $i\in[k]$.
	This is not always the case,
	but we conjecture based on the experimental
	evidence that one of the atoms is recovered with high probability 
	when $k \le p^2/4$ (see Figure~\ref{fig-pt}) and prove a weaker result
	(Theorem~\ref{thm:main}).
	The matrix $G\in\R^{p\times p}$ determines which of
	the atoms $d_id_i^{\top}$ is the optimizer
	and its choice is discussed when we construct a deflation procedure 
	(Section~\ref{sec-deflation}; see also Appendix~\ref{sec-choose-g}).

	\paragraph{Intuition.}
	Since generalized covariances $H_1$,\dots,$H_s$ 
	span the subspace $W$,
	the constraint set of~\eqref{def-sdp} is:
	\begin{equation}
	\label{def-kcal}
	\setlength\abovedisplayskip{.5mm}
	\kcal := \cbra{B \in W : \tr(B) = 1, B \succeq 0}.
	\vspace{-2mm}
	\end{equation}
	It is not difficult to show (see Appendix~\ref{app-proof-lemma-extreme-points})
	that under Assumption~\ref{ass-lnd}
	the atoms $d_id_i^{\top}$ are extreme
	points of this set $\kcal$:
	\begin{lemma}
		\label{lem-extreme-points}
		Let the atoms $d_1d_1^{\top}$, $d_2d_2^{\top}$, $\dots$, $d_kd_k^{\top}$
		be linearly independent. Then they are extreme points
		of the set $\kcal$ defined in~\eqref{def-kcal}.
	\end{lemma}
	
	If the program~\eqref{def-sdp} has a unique solution,
	the optimizer $B^{\ast}_{sdp}$ must be an extreme point
	due to the compactness of the convex set $\kcal$.
	If the set~\eqref{def-kcal} does not have other extreme points
	except for the atoms $d_id_i^{\top}$, $i\in[k]$,
	then the optimizer is guaranteed to be one of the atoms.
	This might not be the case if the set $\kcal$ contains extreme
	points different from the atoms.
	This might explain why the phase transition (at the rate $k\le p^2/4$)
	happens and could potentially be related
	to the phenomenon of 
	polyhedrality of spectrahedra\footnote{
		The spectrahedron is a set formed by an intersection 
		of the positive semi-definite cone with linear constraints, e.g. the set $\kcal$.
		Importantly, all polyhedra are spectrahedra, but not all spectrahedra are polyhedra. 
	}
	\citep{BhaEtAl2015}.
	
	Before diving into the analysis of this SDP, let us
	present its convex relaxation which enjoys certain
	desirable properties.

\subsubsection{The Convex Relaxation}
\label{sec-relaxation}
	
	Let us rewrite~\eqref{def-sdp} in an equivalent form.
	The constraint $B \in W := \spn\cbra{d_1d_1^{\top}, \dots, d_kd_k^{\top}}$
	is equivalent to the fact that $B$ is orthogonal to any matrix from
	the orthogonal complement (null space) of $W$. 
	Let the matrices
	$\cbra{F_1, F_2, \dots, F_{m-k}}$, where $m:=p(p+1)/2$,
	form a basis of the null space $\ncal(W)$.\footnote{
		Note that a basis of $\ncal(W)$ can be easily computed 
		given matrices $H_1,\dots,H_s$.
	} 
	Then the program~\eqref{def-sdp}
	takes an equivalent formulation:
	\begin{equation}
		\label{def-sdp-equiv}
		\setlength\abovedisplayskip{.5mm}
		\begin{aligned}
			B^{\ast}_{sdp} := &\argmax_{B\in\scal_p} \; \inner{G,B} \\
			& \inner{B,F_j} = 0, \quad \text{for all}\quad j\in[m-k], \\
			& \tr(B) = 1, \\
			& B \succeq 0.
		\end{aligned}
		\vspace{-4mm}
	\end{equation}
	
	In the presence of (e.g., finite sample) noise, 
	the subspace $W$ can only be estimated approximately
	(in the first step).
	Therefore, rather than keeping the hard first constraint, we introduce the 
	relaxation
	\begin{equation}
		\label{def-sdp-relax}
		\setlength\abovedisplayskip{.5mm}
		\begin{aligned}
			B^{\ast} := \argmax_{B\in\scal_p} & \; \inner{G,B} - \frac{\mu}{2} \sum_{j\in[m-k]} \inner{B,F_j}^2 \\
			& \tr(B) = 1, \; B \succeq 0,
		\end{aligned}
		\vspace{-3mm}
	\end{equation}
	where $\mu>0$ is a regularization parameter
	which helps to adjust to an expected level of noise.
	Importantly,
	the relaxation~\eqref{def-sdp-relax}
	can be solved efficiently, e.g., via the fast iterative shrinkage-thresholding algorithm \citep[FISTA;][]{BecTeb2009}
	and the majorization-maximization principle
	\citep[see, e.g.,][]{HutLan2004}.
	See Appendix~\ref{app-sdp-fista} for details.

\subsubsection{Deflation}
\label{sec-deflation}
	
	The semi-definite program~\eqref{def-sdp},
	or its relaxation~\eqref{def-sdp-relax}, 
	is designed to estimate only some one atom $d_id_i^{\top}$.
	To estimate all other atoms we need a deflation procedure.
	In general, there is no easy and straightforward way 
	to perform deflation in the overcomplete setting,
	but we discuss possible approaches below.

	\paragraph{Clustering.}
	Since the matrix $G$ determines which atom is found,
	it is natural to repeatedly resample this matrix 
	a multiple of $k$ times and then cluster the obtained
	atoms into $k$ clusters. This approach generally works well
	except in the cases where either (a) some of the atoms,
	say $d_id_i^{\top}$ and $d_jd_j^{\top}$,
	are relatively close (e.g., in terms of angle in the 
	space of all symmetric matrices)
	to each other,
	or (b) one or several atoms were not properly
	estimated.
	In the former case, one could increase the number of 
	times $G$ is resampled, and the program is solved,
	but that might require very high number of repetitions.
	The latter issue is more difficult to fix since 
	a single wrong atom could significantly perturb
	the overall outcome.

	\paragraph{Adaptive Deflation.}
	Alternatively, one could adapt the constraint set iteratively to exclude from the search
	all the atoms found so far.
	For that, one can update the constraint set so that the subspace
	$W$ 
	is replaced with the subspace that is
	spanned by all the atoms except for the ones which were already found.
	The most natural way to implement this is to add the found
	atoms to a basis of the null space of $W$, which is straightforward
	to implement with the relaxation~\eqref{def-sdp-relax}.
	Similar to other deflation approaches,
	a poor estimate of an atom obtained in an earlier
	deflation step 
	of such adaptive deflation
	can propagate this error
	leading to an overall poor result.

	\paragraph{Semi-Adaptive Deflation.}
	We found that taking advantage of both presented
	deflation approaches leads to the best result in practice.
	In particular, we combine these approaches by first performing
	clustering and keeping only good clusters (with low variance over the cluster) and then continuing with the adaptive deflation approach.
	We assume this \textbf{semi-adaptive deflation} approach 
	for all the experiments presented in Section~\ref{sec-exps}.

\subsubsection{Identifiability}
\label{sec-theory}
	
	In general, there are two types of identifiability 
	of probabilistic models:
	(a) statistical and (b) algebraic.
	The \emph{statistical} identifiability addresses whether the parameters of the model can be
	identified for given distributions. In particular, it is well known 
	that the ICA model is not identifiable if (more than one of) the sources are Gaussian
	\citep{Com1994} and issues also arise when the sources
	are close to Gaussian \citep{SokEtAl2014}. 
	These results also extend to the overcomplete case 
	that we consider.
	However, we do not address
	these questions here and assume that the models we work with are statistically
	identifiable. Instead, we are interested whether our approach is \emph{algebraically}
	identifiable, i.e., whether our algorithm correctly recovers the 
	parameters of the model.
	In particular, we address the following question:
	\emph{When is the solution $B^{\ast}_{sdp}$ of the program~\eqref{def-sdp}
		is one of the atoms $d_id_i^{\top}$, $i\in[k]$?}
	
	We address this question in theory and in practice and focus on the
	population (infinite number of samples) case, where we assume that
	an exact estimate of the subspace $W$ is given
	and, therefore, one can use the representation
	$W:= \spn\cbra{d_1d_1^{\top}, \dots, d_kd_k^{\top}}$
	without loss of generality.
	Therefore, for the theoretical analysis purposes we assume that atoms $d_id_i^{\top}$ are known,
	we consider the following program instead
	\begin{equation}
		\label{def-sdp-theory}
		\setlength\abovedisplayskip{.5mm}
		\begin{aligned}
			B^{\ast}_{sdp} := &\argmax_{B\in\scal_p} \; \inner{G,B} \\
			& B \in \spn\cbra{d_1d_1^{\top}, d_2d_2^{\top}, \dots, d_kd_k^{\top}}, \\
			& \tr(B) = 1, \\
			& B \succeq 0.
		\end{aligned}
		\vspace{-5mm}
	\end{equation}
	
	\paragraph{Phase Transition.}
	\begin{figure}[t]
		\centering
		\includegraphics[width=0.7\linewidth]{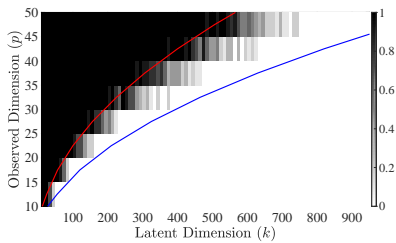}
		\vspace{-3mm}
		\caption{Phase transition of the program~\eqref{def-sdp-theory}.}
		\label{fig-pt}
		\vspace{-5mm}
	\end{figure}
	
	In Figure~\ref{fig-pt}, we present the phase transition plot
	for the program~\eqref{def-sdp-theory} obtained by solving
	the program multiple times for different settings.
	In particular, for every pair $(p,k)$ we solve the program
	$n_{rep}:=50$ times and assign to the respective point
	the value equal to the fraction of successful solutions
	(where the optimizer was one of the atoms).
	
	Given a fixed pair $(p,k)$, every instance of the program~\eqref{def-sdp-theory}
	is constructed as follows. We first sample a mixing matrix
	$D\in\R^{p\times k}$
	so that every mixing component is from the standard
	normal distribution as described in 
	Appendix~\ref{app-sampling}; and we sample a matrix 
	$G\in\R^{p\times p}$ from the standard normal distribution.
	We then construct the constraint set of the program~\eqref{def-sdp-theory}
	by setting every matrix $H_i = d_id_i^{\top}$ for all $i\in[k]$,
	where $s=k$.
	We solve every instance of this problem with the CVX toolbox \citep{GraEtAl2006} using
	the SeDuMi solver \citep{SeDuMi}.

	We consider the observations dimensions $p$ from $10$ to $50$ with 
	the interval of $5$ and we vary the number of atoms
	from $10$ to $1000$ with the interval of $10$.
	The resulting phase transition plots are presented
	in Figure~\ref{fig-pt}.
	The \textbf{blue line} on this plot corresponds to the curve
	$k=p(p+1)/2$, which is the largest possible latent dimension
	of all symmetric matrices $\scal_p$.
	The \textbf{red line} on this plot corresponds to the curve
	$k = p^2/4$. Since above the red line
	we observe 100\% successful recovery (black),
	\emph{we conjecture that the phase transition happens
		around $k=p^2/4$.}

	\paragraph{Theoretical Results.}
	Interestingly,
	an equivalent conjecture, $k < p^2/4$, was made for the ellipsoid fitting problem
	\citep{SauEtAl2012,SauEtAl2013} and the question remains open to our best knowledge.\footnote{
		In Appendix~\ref{app-ellipsoid-fitting},
		we recall the formulation of the ellipsoid fitting problem and
		slightly improve the results of \citet{SauEtAl2012,SauEtAl2013}.
	}
	In fact, we show close relation between successful solution (recovery of an atom)
	of our program~\eqref{def-sdp-theory} and the ellipsoid fitting problem.
	In particular, a successful solution of our problem implies
	that the feasibility of its Lagrange dual program
	is equivalent to the ellipsoid fitting problem
	(see Appendix~\ref{app-dual}).
	Moreover, using this connection, 
	we prove the following:
	\begin{theorem}\label{thm:main}
		Let $\varepsilon > 0$. Consider a regime with $p$ tending to to infinity, and with $k$ varying according to the bound $k < (2 - \varepsilon) p \log p$. As above, let the $d_i$ be random unit vectors and let $G = uu^\top$ for a random unit vector $u$. Then with high probability\footnote{Throughout, ``with high probability'' indicates probability tending to $1$ as $p \to \infty$.}, the matrix $d_i d_i^\top$ for which $d_i^\top G d_i$ is largest is the unique maximizer of the program~\eqref{def-sdp-theory}. 
	\end{theorem}
	
	\begin{figure*}[t]
		\centering
		\begingroup
		\setlength{\tabcolsep}{0pt} 
		\renewcommand{\arraystretch}{.5} 
		\begin{tabular}{c c c c}
			\includegraphics[width=.29\textwidth]{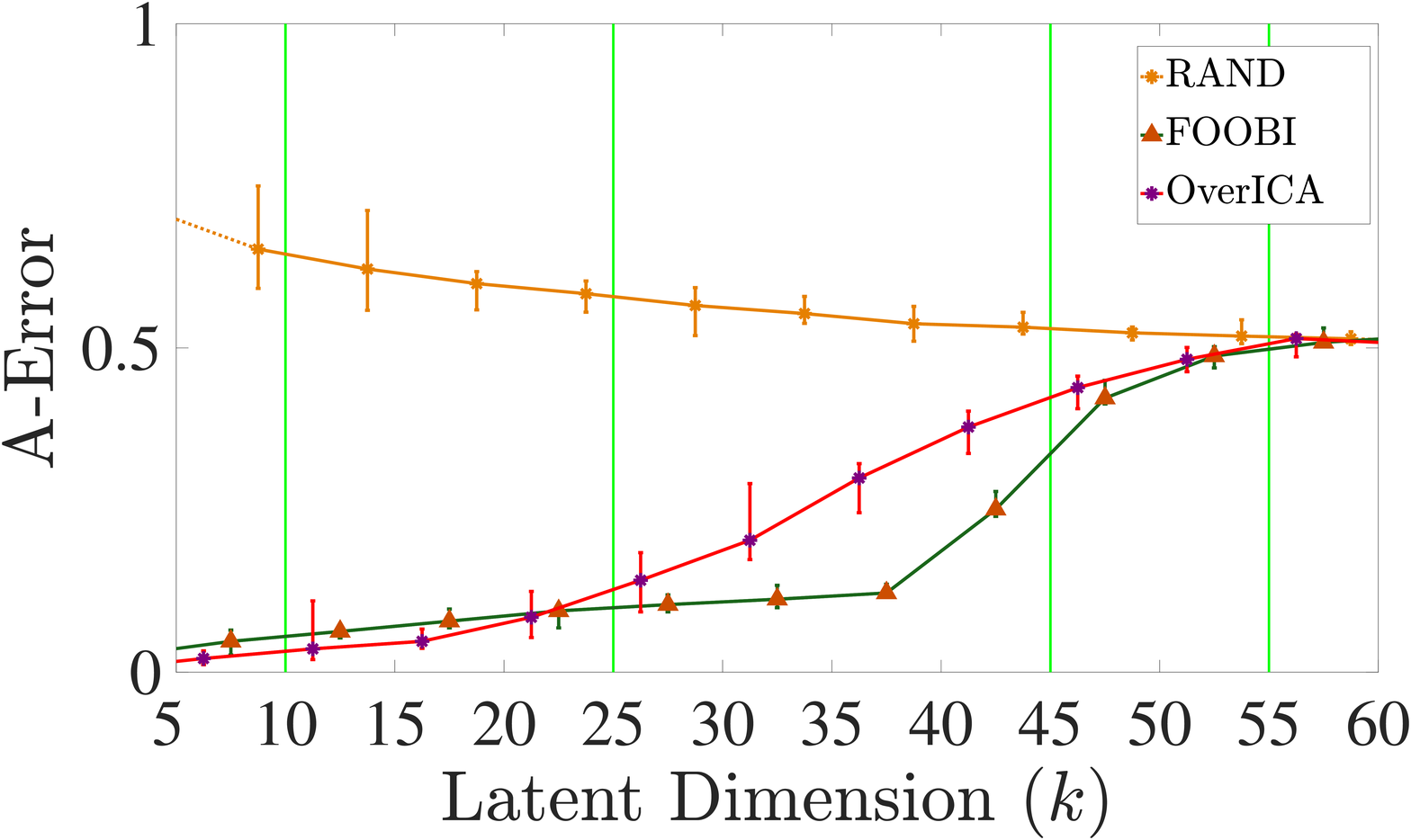}
			&
			\includegraphics[width=.29\textwidth]{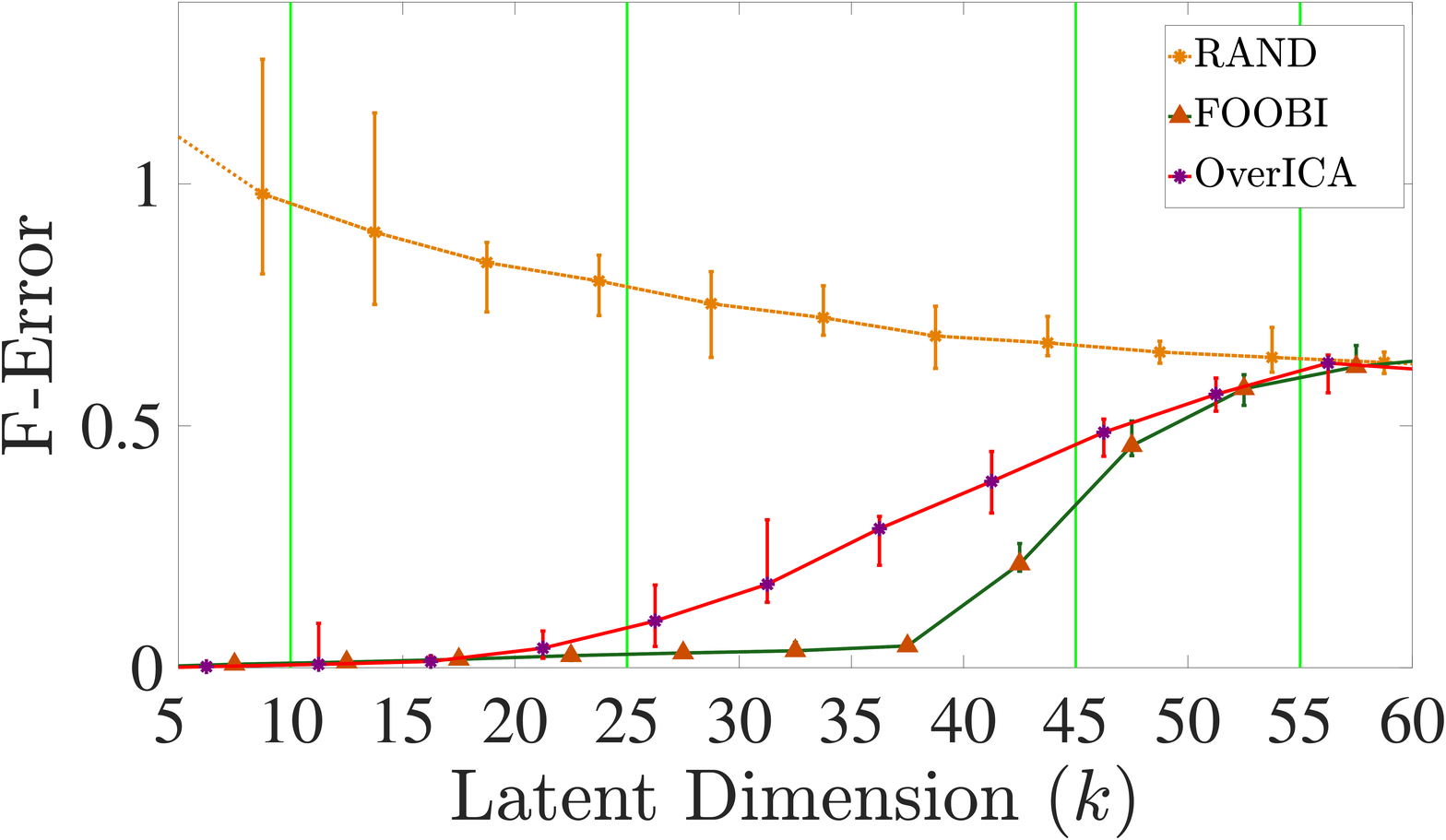}
			&
			\includegraphics[width=.19\textwidth]{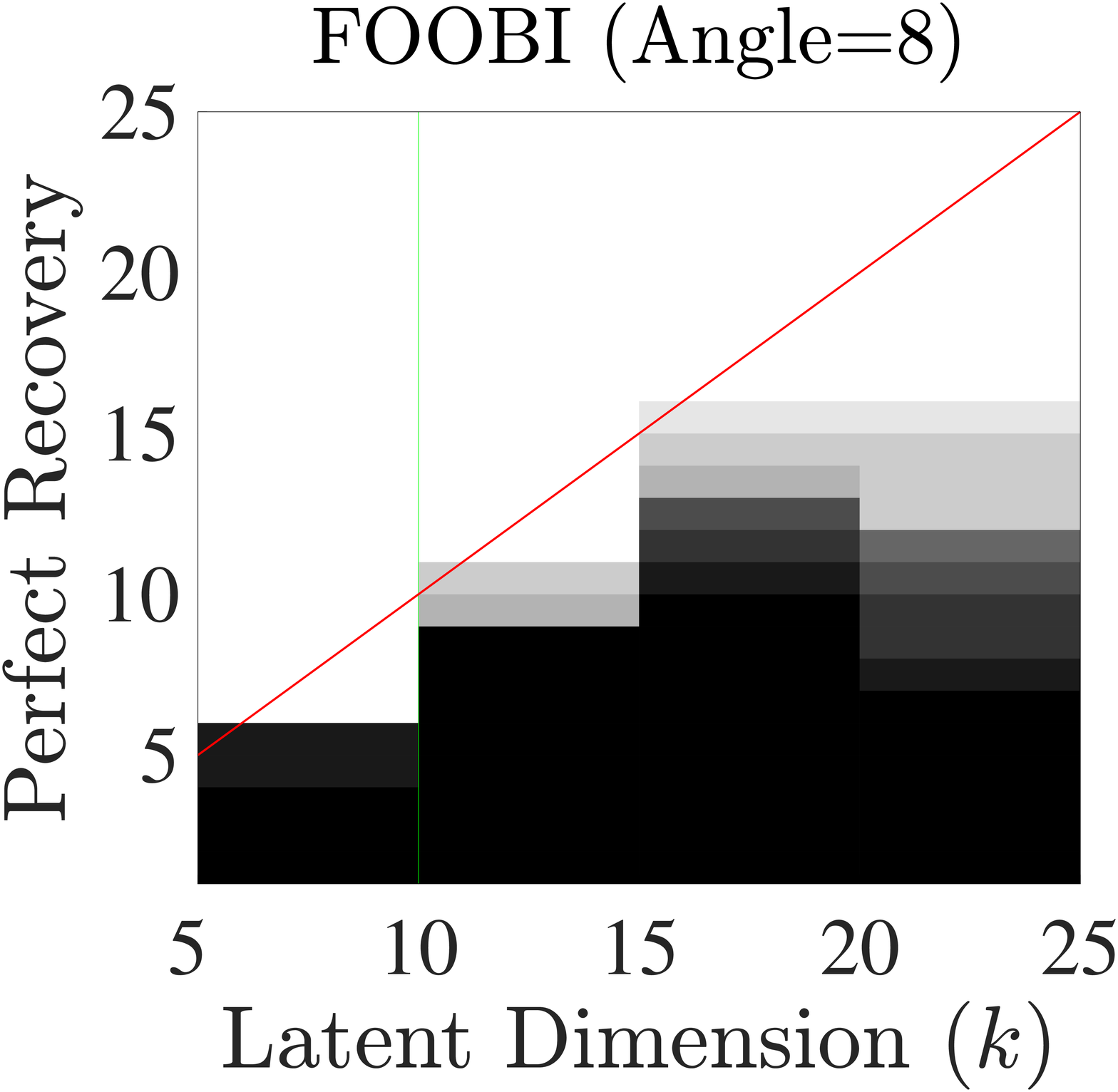}
			&
			\includegraphics[width=.19\textwidth]{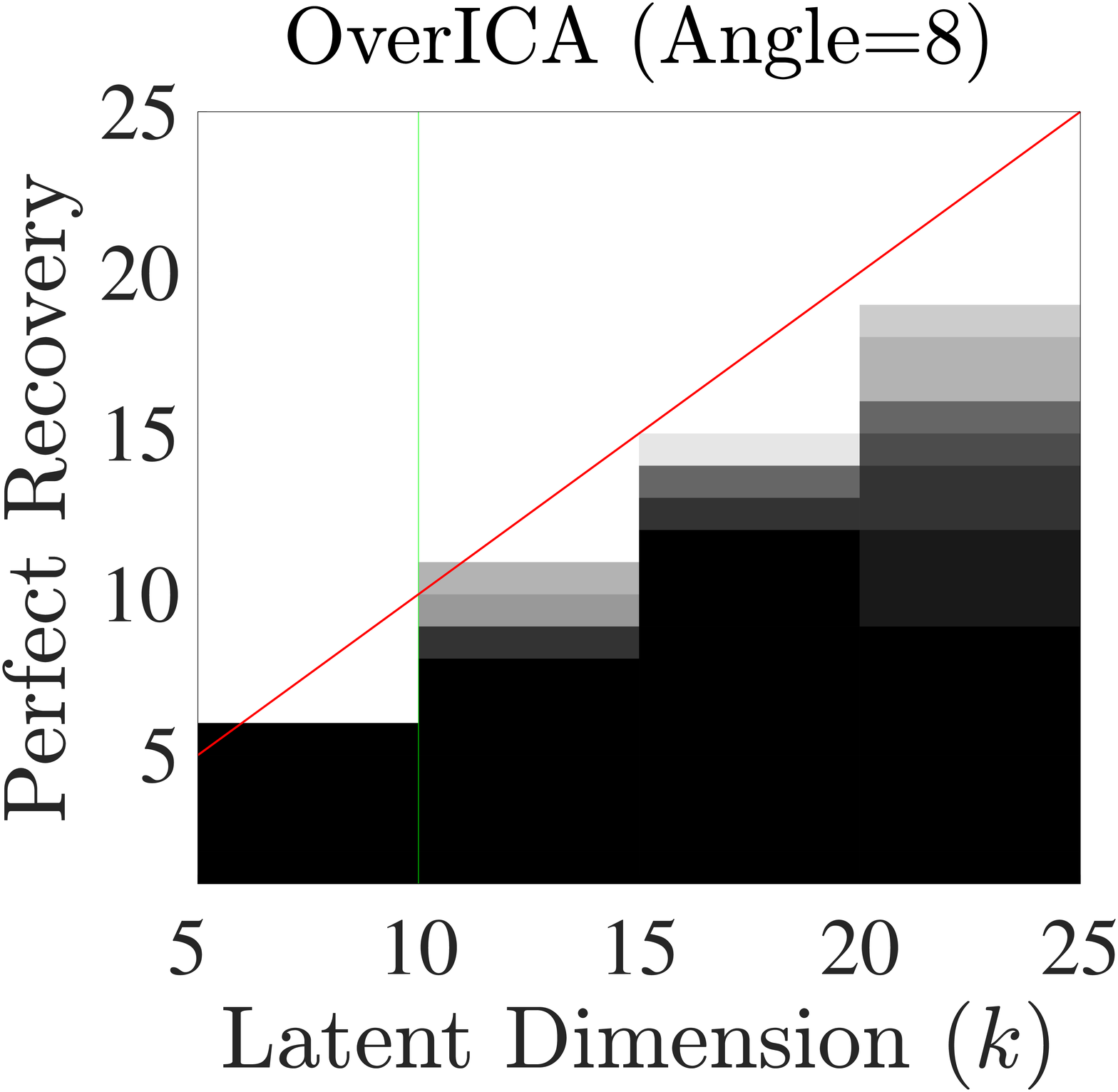}
		\end{tabular}
		\endgroup
		\vspace{-3mm}
		\caption{
			A proof of concept in the asymptotic regime.
			See explanation in Section~\ref{sec-pexps}.
		}
		\label{fig-pop}
		\vspace{-2mm}
	\end{figure*}
	
	\begin{figure*}[!t]
		\centering
		\begingroup
		\setlength{\tabcolsep}{0pt} 
		\renewcommand{\arraystretch}{.5} 
		\begin{tabular}{c@{\hskip 2mm} c@{\hskip 2mm} c@{\hskip 2mm} c}
			\includegraphics[width=.23\textwidth]{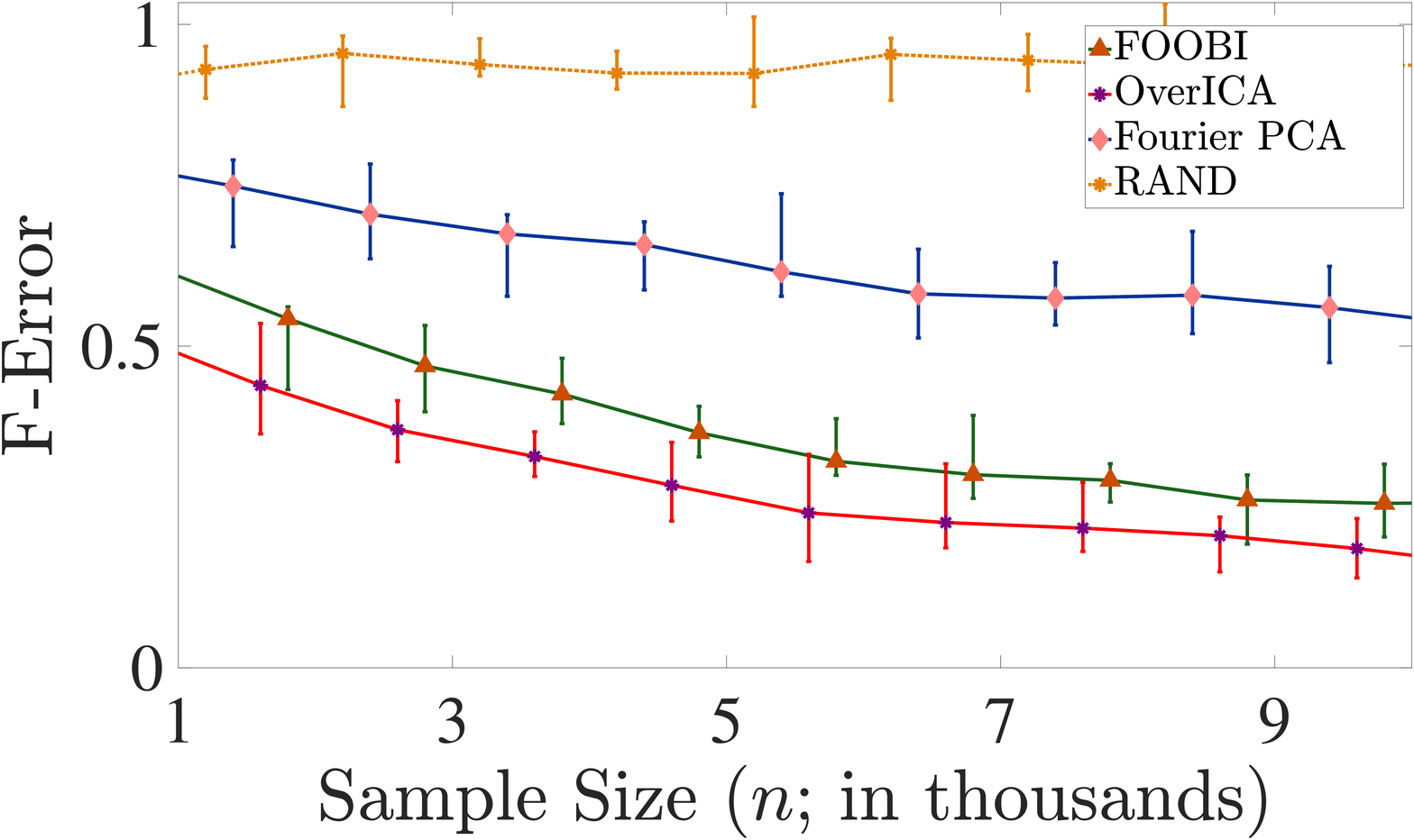}
			& 
			\includegraphics[width=.23\textwidth]{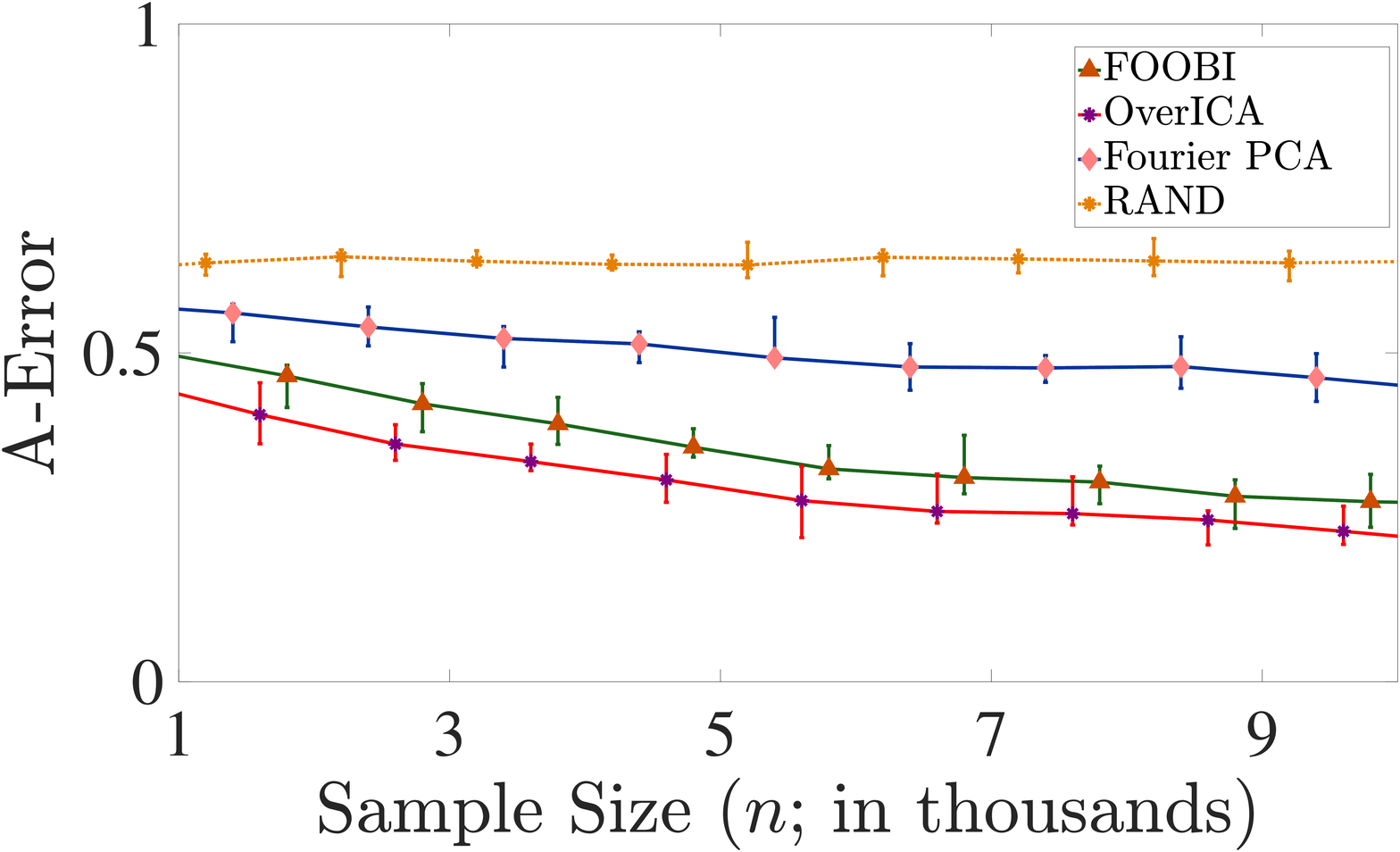}
			&
			\includegraphics[width=.23\textwidth]{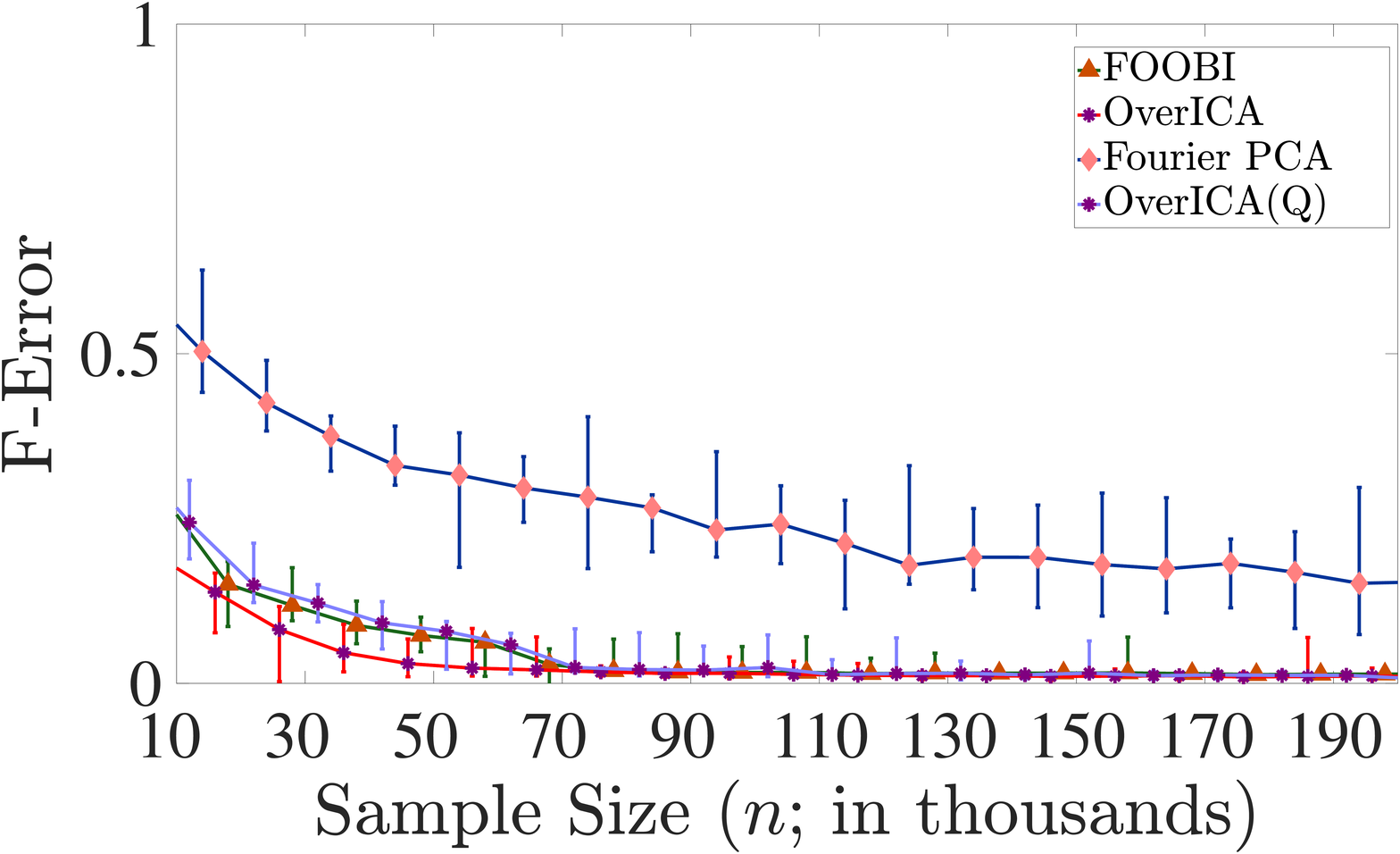}
			&
			\includegraphics[width=.23\textwidth]{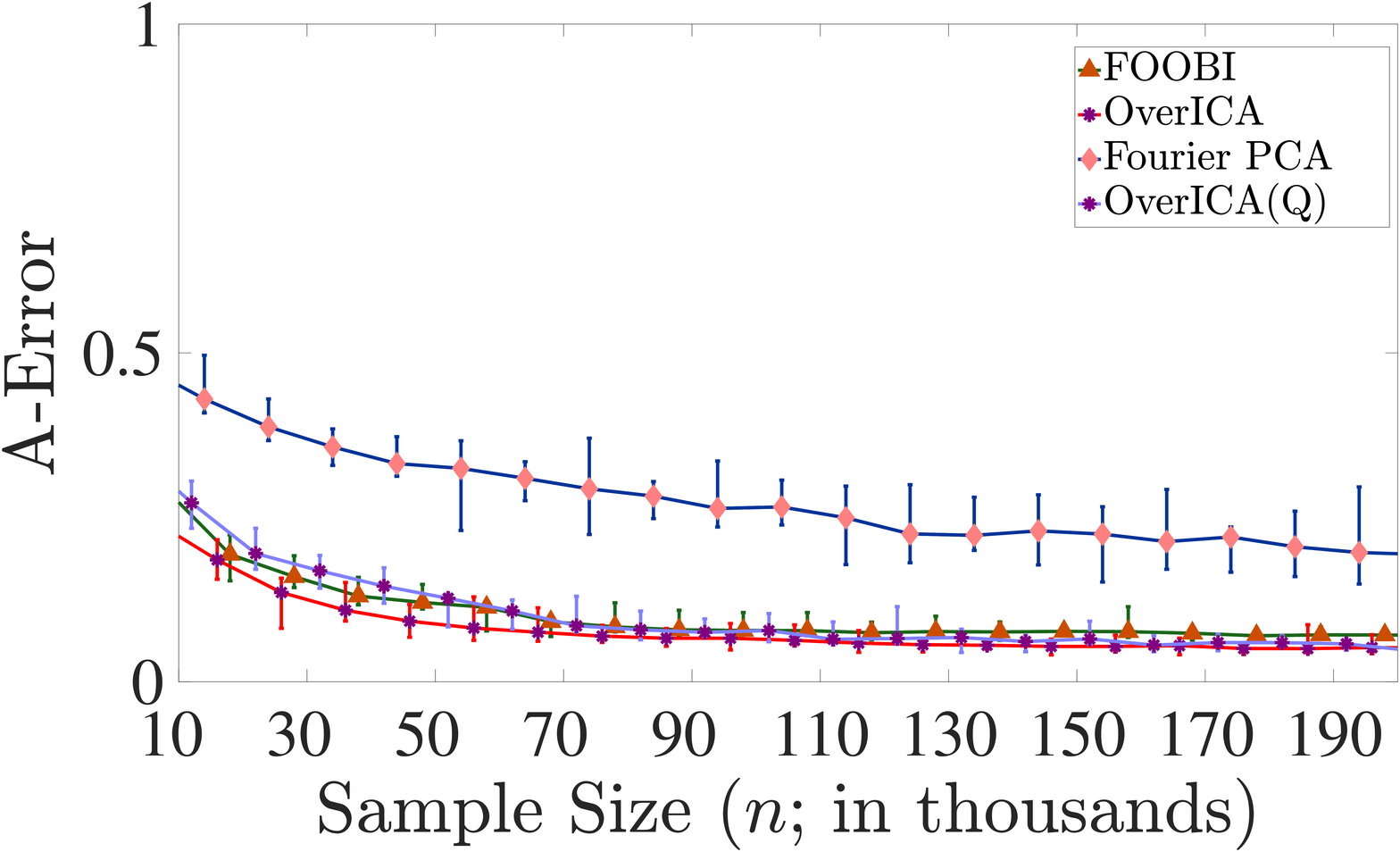}
			\\
			\includegraphics[width=.23\textwidth]{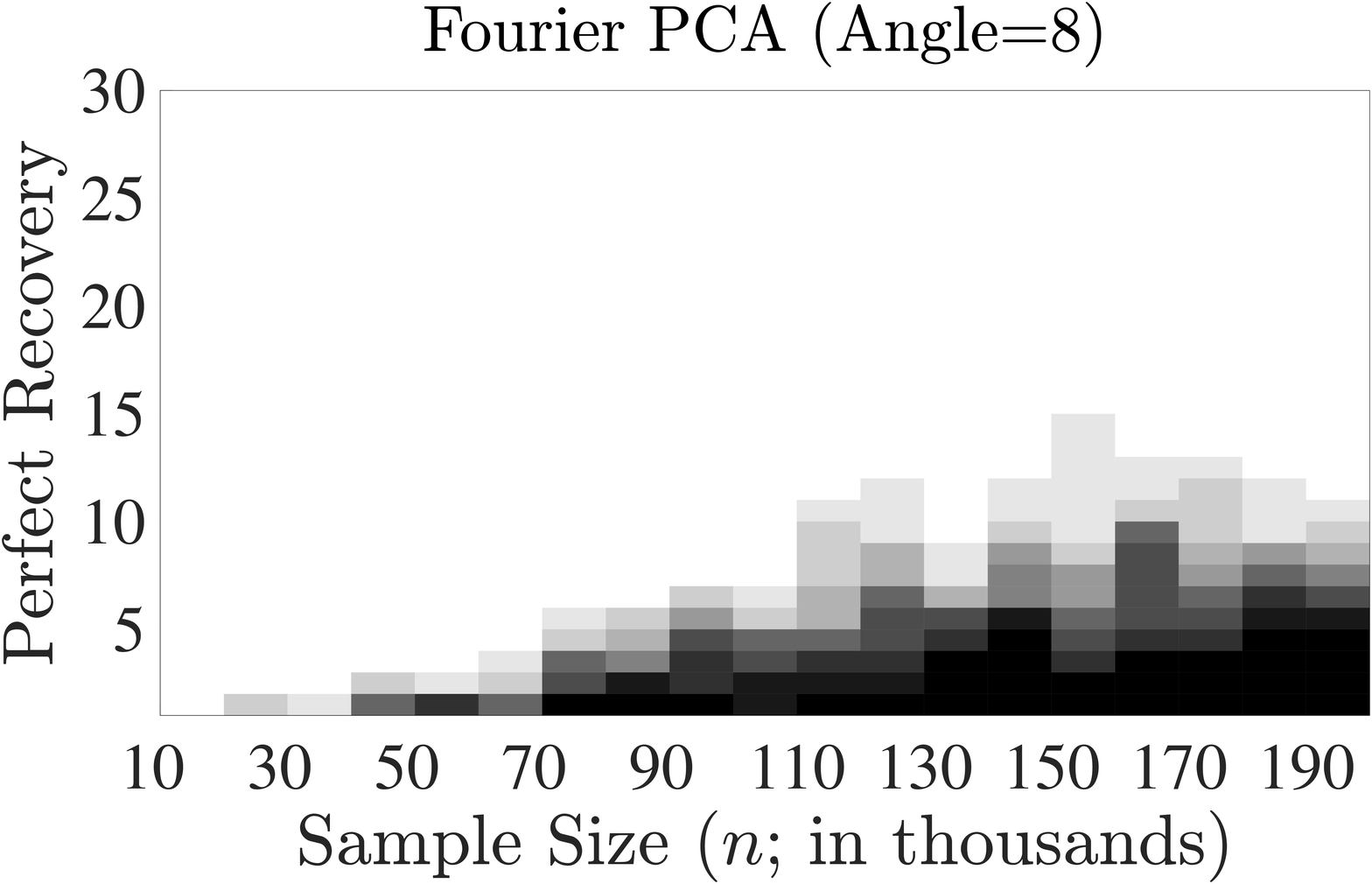}
			&
			\includegraphics[width=.23\textwidth]{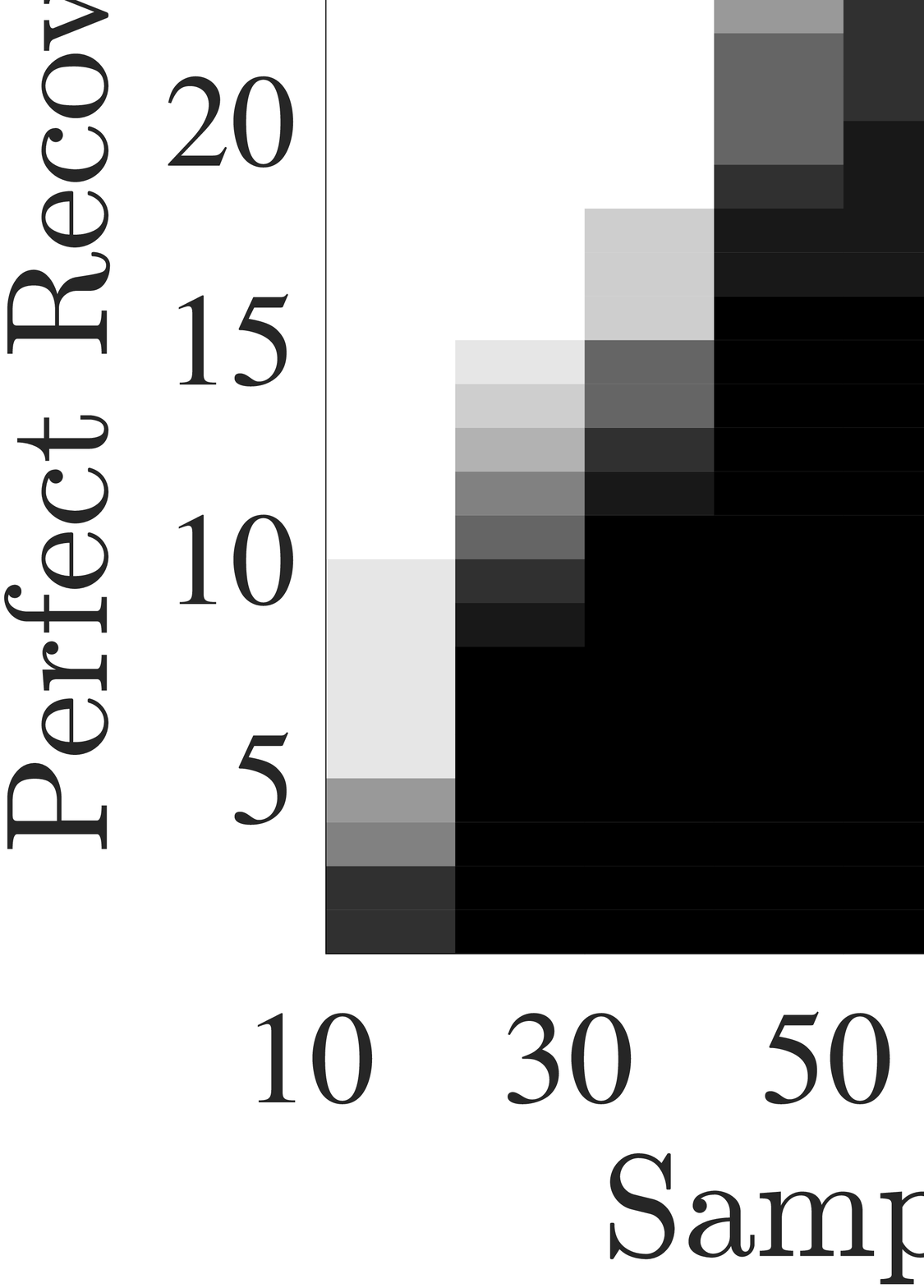}
			&
			\includegraphics[width=.23\textwidth]{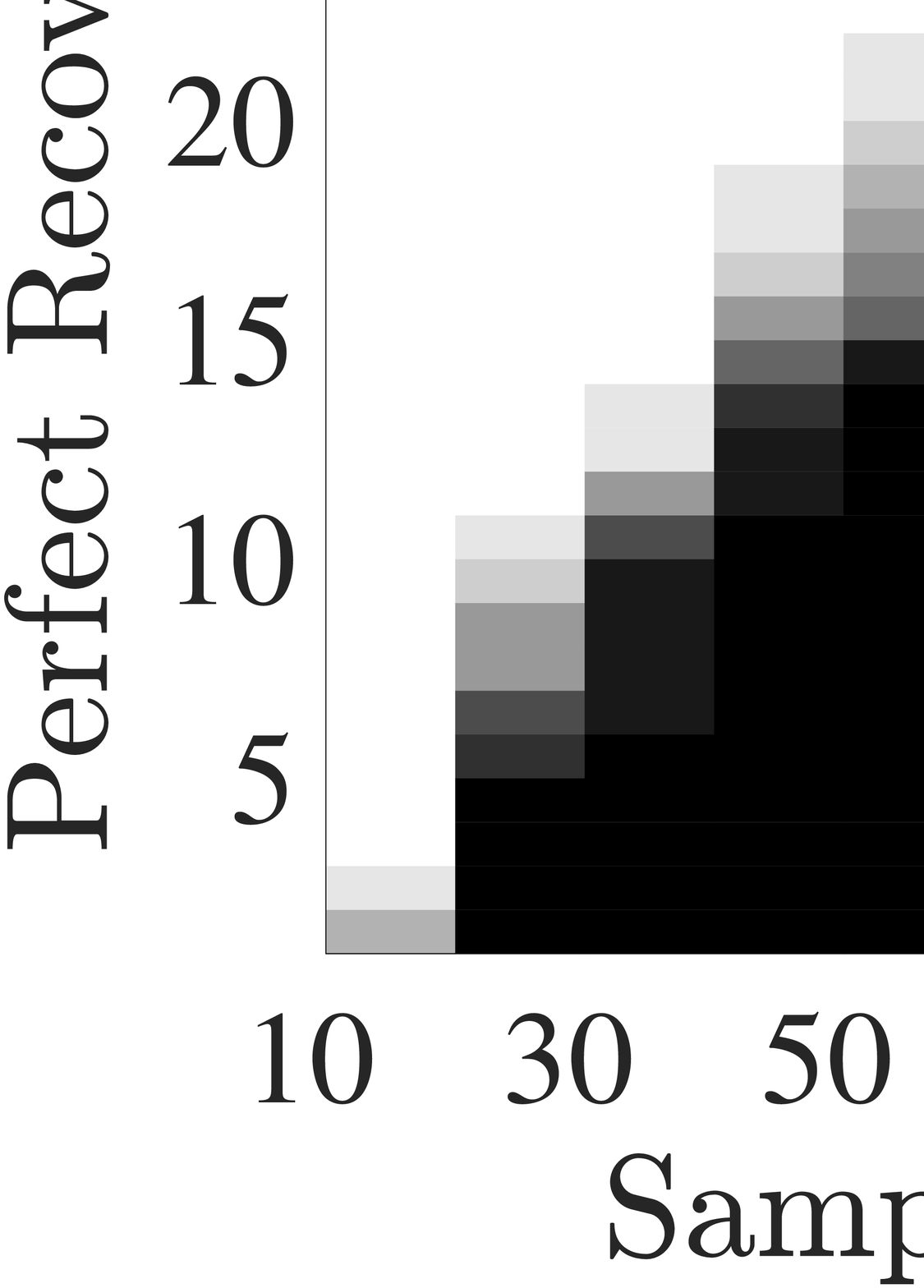}
			&
			\includegraphics[width=.23\textwidth]{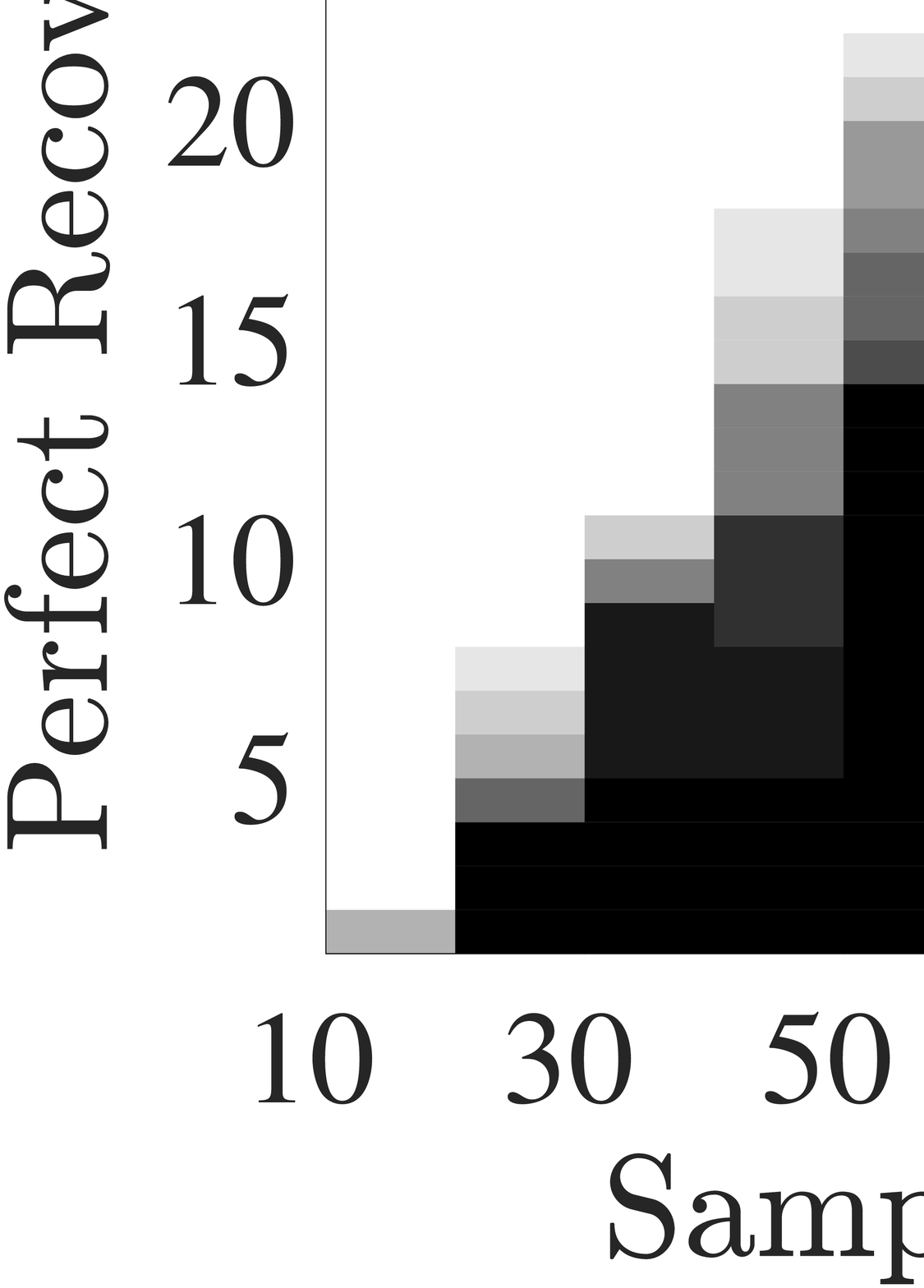}
		\end{tabular}
		\endgroup
		\vspace{-3mm}
		\caption{
			Comparison in the finite sample regime.
			See explanation in Section~\ref{sec-exps-fs}.
		}
		\vspace{-.4cm}
		\label{fig-fs}
	\end{figure*}

	\section{Experiments}
	\label{sec-exps}
	
	It is difficult to objectively evaluate
	unsupervised learning algorithms
	on real data in the absence of ground truth parameters.
	Therefore, we first perform comparison on synthetic data.
	All our experiments can be reproduced with the publicly available code:
	\url{https://github.com/anastasia-podosinnikova/oica}.

\subsection{Synthetic Data: Population Case}
\label{sec-pexps}
	
	As a proof of concept, this simple experiment (Figure~\ref{fig-pop})
	imitates the infinite sample case. 
	Given a ground truth mixing matrix $D$,
	we construct a basis of the subspace $W$
	directly from
	the matrix $A := D\odot D$ (see Appendix~\ref{app-ica-cum4}).
	This leads to a noiseless estimate of the subspace.
	We then evaluate the performance of the second step
	of our OverICA algorithm and compare it with the 
	second step of FOOBI.
	We fix the
	observed dimension $p=10$ and vary the latent dimension from $k=5$ to $k=60$ in steps of
	$5$. For every pair $(p,k)$, we repeat the experiment
	$n_{rep}=10$ times and display the minimum, median, 
	and maximum values. Each time we sample
	the mixing matrix $D$ with mixing components
	from the standard normal distribution
	(see Appendix~\ref{app-samling-mm}).
	Note that we tried
	different sampling methods and distributions of the
	mixing components, but did not observe any significant
	difference in the overall result. 
	See Appendix~\ref{app-sampling-population} for further
	details on this sampling procedure.
	
	The error metrics (formally defined in  Appendix~\ref{app-errors}) are:
	(a) f-error
	 is essentially the relative Frobenius
	norm of the mixing matrices with properly
	permuted mixing components (lower is better);
	(b) a-error measures the angle
	deviations of the estimated mixing components
	vs the ground truth (lower is better); and (c) ``perfect'' recovery 
	rates, which show for every $i\in[k]$ the fraction
	of perfectly estimated $i$ components. 
	We say that a mixing component is ``perfectly'' recovered
	if the cosine of the angle between this component $d_i$ and its ground truth equivalent $d_{\pi(i)}$ is at least $0.99$, i.e., $\cos(d_i,\wh{d}_{\pi(i)}) \ge 0.99$. Note that the respective angle is approximately equal to 8.
	Then the black-and-white perfect recovery
	plots (in Figure~\ref{fig-pop})
	show if $i\le k$ (on the y-axis)
	components were perfectly recovered
	(black)
	for the given latent dimension~$k$
	(x-axis). These black vertical bars
	cannot exceed the red line $i=k$,
	but the closer they approach this line, 
	the better.
	The vertical green lines correspond to 
	$k=p=10$, $k=p^2/4=25$, $k=p(p-1)/2$, 
	and $k=p(p+1)/2$.
	Importantly, we see that  OverICA
	 works better or comparably 
	to  FOOBI  in the regime $k< p^2/4$. 
	Performance of OverICA 
	starts to deteriorate near the regime 
	$k \approx p^2/4$ and beyond, which is in accord with 
	our theoretical results in Section~\ref{sec-theory}.
	Note that to see whether the algorithms work better than random,
	we display the errors of a randomly sampled mixing matrix
	(RAND; see Appendix~\ref{app-samling-mm}).
	
	\subsection{Synthetic Data: Finite Sample Case}
	\label{sec-exps-fs}
	
	With these synthetic data we evaluate performance of overcomplete
	ICA algorithms in the presence of finite sample noise 
	but absence of model misspecification.
	In particular, we sample synthetic data in the observed dimension $p=15$ from
	the ICA model
	with uniformly distributed (on $[-0.5, 0.5]$) $k=30$
	sources for different sample sizes $n$ 
	taking values from 
	$n=1,000$ to $n=10,000$
	in steps of $1,000$ (two left most plots in the top line of Figure~\ref{fig-fs})
	and values from
	$n=10,000$ to $n=210,000$
	in steps of $10,000$
	(two right most plots in the top line of Figure~\ref{fig-fs};
	see also Figure~\ref{fig-fs-more} in Appendix for log-linear scale).
	Note that the choice of dimensions $p=15$ and $k=30$
	corresponds to the regime $k< p^2/4\approx56$
	of our guarantees.
	We repeat the experiment $n_{rep} :=10$ times
	for every $n$ where we every time resample 
	the (ground truth) mixing matrix 
	(with the sampling procedure described in 
	Appendix~\ref{app-samling-mm}).
	See further explanation 
	in Appendix~\ref{app-sampling-finite}.
	
	We compare the Fourier PCA algorithm 
	\citep{GoyEtAl2014},
	the FOOBI algorithm \citep{LatEtAl2007},
	OverICA from Algorithm~\ref{alg-overica},
	and a version of the OverICA algorithm 
	where the first step is replaced with the 
	construction based on the fourth-order cumulant, a.k.a.
	quadricovariance
	(OverICA(Q); see Appendix~\ref{app-ica-cum4}).
	Note that we can not compare with the reconstruction ICA
	algorithm by \citet{LeEtAl2011} because it estimates the de-mixing
	(instead of mixing) matrix.\footnote{
		In the complete invertible case, the \emph{de-mixing matrix} would be the 
		inverse of the mixing matrix. In the overcomplete regime, one cannot
		simply obtain the mixing matrix from the de-mixing matrix.
	}
	Similarly to Section~\ref{sec-pexps},
	we measure the Frobenius error (f-error),
	the angle error (a-error), and the perfect recovery
	for the angle of $8$.
	We observe that the generalized covariance-based OverICA
	algorithm performs slightly better which we believe is
	due to the lower sample complexity.
	Fourier PCA on the contrary performs with larger error,
	which is probably due to the higher sample complexity
	and larger noise resulting from estimation using  
	fourth-order generalized cumulants.
	
	\begin{table}[t]
	\caption{
		Computational complexities
		($n$ is the sample size, $p$ is the observed dimension, $k$ is the latent dimension,
		$s$ is the number of generalized covariances, usually $s=O(k)$).	
	} 
	\vspace{-2mm}
	\label{tab-compl}
	\begin{center}
		\begin{tabular}{| l | l | l |}
			\hline
			Procedure & Memory & Time \\
			\hline
			GenCov          & $O(p^2s)$ & $O(snp^2)$ \\
			CUM               & $O(p^4)$   & $O(np^4+k^2p^2)$ \\ \hline
			FOOBI             & $O(p^4k^2 + k^4)$ & $O(np^4+k^2p^4+k^6)$ \\
			OverICA          & $O(sp^2)$ & $O(nsp^2)$ \\
			OverICA(Q)     & $O(p^4)$   & $O(np^4+k^2p^2)$ \\
			Fourier PCA    & $O(p^4)$   & $O(np^4)$\\
			\hline
		\end{tabular}
		\vspace{-5mm}
	\end{center}
\end{table}
\subsection{Computational Complexities}
\label{sec-runtime}
	
	In Table~\ref{tab-compl}, we summarize the timespace complexities of the considered overcomplete
	ICA algorithms and two sub-procedures they use:
	generalized covariances (GenCov; used by OverICA) from Section~\ref{sec-subspace} and the forth-order cumulant (CUM; used by OverICA(Q) and FOOBI; see Appendix~\ref{app-ica-cum4}) 
	(see Appendix~\ref{app-compl}).
	Importantly, we can see that our OverICA algorithm has a significantly lower complexity. 
	In Appendix~\ref{app-compl}, we present runtime comparisons of these algorithms.

\subsection{Real Data: CIFAR-10 Patches}
	Finally, we estimate the overcomplete mixing matrix
	of data formed of patches of the 
	CIFAR-10 dataset \citep[see, e.g.,][]{KriEtAl2014}.
	In particular, we
	transform the images into greyscale and then 
	form $7$-by-$7$ patches for every 
	interior point (at least 3 pixels from the boundary)
	of every image from the training batch~1
	of the CIFAR-10 dataset. 
	This results in $6,760,000$ patches each of dimension $p=49$.
	We perform the estimation 
	of the mixing matrix for $k=150$ latent mixing components.
	The resulting atoms are presented in Figure~\ref{fig-cifar-k150}.
	Note that since ICA is scale (and therefore sign) invariant,
	the sign of every component can be arbitrary flipped.
	We present the obtained components in the scale where black and white
	corresponds to the extreme positive or negative values
	and we observe that these peaks are concentrated in rather pointed areas
	(which is a desirable property of latent components).
	Note that the runtime of this whole procedure was  around~$2$ hours
	on a laptop.
	Due to high timespace complexities (see Section~\ref{sec-runtime}),
	we cannot perform similar estimation neither with FOOBI nor with Fourier
	PCA algorithms.

\begin{figure}[t]
	\vspace{-2mm}
	\hspace{-5mm}\includegraphics[width=1.18\linewidth,left,trim= 150 0 0 0,clip]{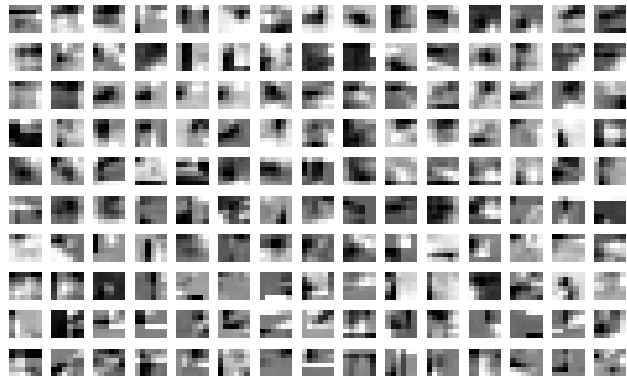}
	\vspace{-8mm}
	\caption{
		Mixing components obtained
		from $7$-by-$7$ patches, i.e., $p=49$, 
		of the CIFAR-10 dataset ($k=150$, i.e., overcomplete).
		ICA does not preserve non-negativity and the 
		signs of ICA mixing components can be arbitrarily flipped
		due to the scaling unidentifiability;
		here black and white correspond to the extreme positive and extreme negative values. 
		The colorbar limits of every image are the same and the signs are alligned to have positive scalar product
		with the first component.
	}
	\vspace{-5mm}
	\label{fig-cifar-k150}
\end{figure}

\section{Conclusion}
	We presented a novel ICA algorithm for
	estimation of the latent overcomplete 
		mixing matrix.
	Our algorithm also works in the (under-)complete setting, enjoys lower computational
	complexity, and comes with theoretical guarantees,
	which is also confirmed by experiments.

\section*{Acknowledgements}
	A. Podosinnikova was partially supported by DARPA grant \#W911NF-16-1-0551.
	A. Podosinnikova and D. Sontag were partially supported by NSF CAREER
	award \#1350965. 
	This work was supported in part by NSF CAREER Award CCF-1453261 and a grant from the MIT NEC Corporation.
	Part of this work was done while A. S. Wein was at the Massachusetts Institute of Technology. A. S. Wein received Government support under and awarded by DoD, Air Force Office of Scientific Research, National Defense Science and Engineering Graduate (NDSEG) Fellowship, 32 CFR 168a. A. S. Wein is also supported by NSF grant DMS-1712730 and by the Simons Collaboration on Algorithms and Geometry.

\bibliographystyle{plainnat}
\bibliography{lit}
\cleardoublepage

\begin{appendices}

	\section{Apendix: Technical Details}
	
	\subsection{Matricization and Vectorization}
	\label{app-mat-vec}
	
	One can vectorize a matrix by stacking its columns and
	one can matricize a vector by performing the reverse of
	the vectorization operation. Below we formalize these notions.
	
	\subsubsection{Vectorization}
	
	Given a matrix $X\in\R^{m\times n}$, we define its \emph{vectorization}
	as $x := \vect(X)\in\R^{mn}$ such that
	\begin{equation}
	\label{def-vectorization}
	x_{(i-1)m+j} := X_{ij}, \quad \text{for all} \quad i\in[m], \; j\in[n].
	\end{equation}
	
	We also use in this paper the fact that the vectorization of a
	rank matrix $X := a b^{\top}$ is equal to the Khatri-Rao product
	(See Appendix~\ref{app-khatri-rao}) of the vectors $a$ and $b$,
	i.e.,
	\begin{equation}
	\label{krprod-rank1}
	\vect(ab^{\top}) = a \odot b.
	\end{equation}
	
	\subsubsection{Matricization}
	
	We define the \emph{matricization} operation as the inverse of the vectorization operation, 
	i.e. the matricization of a vector $x\in\R^{mn}$ is a matrix 
	$X := \mat(x) \in\R^{m \times n}$ such that
	\begin{equation}
	\label{def-matricization}
	X_{ij} := x_{(i-1)m + j}, \quad \text{for all} \quad i\in[m], \; j\in[n].
	\end{equation}

	\subsection{The Kharti-Rao Product}
	\label{app-khatri-rao}
	
	The \emph{Khatri-Rao product} of two matrices $A\in\R^{n \times k}$
	and $B\in\R^{m \times k}$, with columns $b_j$,
	for $j\in[k]$, is an $(nm\times k)$-matrix  $A\odot B$ of the form:
	\begin{equation}
	\label{def-khatri-rao}
	A\odot B := \begin{pmatrix}
	A_{11} b_1 & A_{12} b_2 & \dots & A_{1k} b_k \\
	A_{21} b_1 & A_{22} b_2 & \dots & A_{2k} b_k \\
	A_{31} b_1 & A_{32} b_2 & \dots & A_{3k} b_k \\
	\dots & \dots & \dots & \dots \\
	A_{n1} b_1 & A_{n2} b_2 & \dots & A_{nk} b_k
	\end{pmatrix}.
	\end{equation}
	Note that although the Khatri-Rao product of two vectors 
	coincide with the Kronecker product of these vectors,
	the two products are different in general.
	
	Moreover, as we mentioned in Appendix~\ref{app-mat-vec},
	the vectorization of a rank-1 matrix is equal to the 
	Khatri-Rao product of respective vectors.

	\section{Apendix: Independent Component Analysis}
	
	\subsection{The Fourth-Order Cumulant and Kurtosis}
	\label{app-cum-kurt}
	
	\subsubsection{The Fourth-Order Cumulant}
	\label{app-def-cum}
	
	Given a $p$-valued zero-mean random vector $x$, its
	\emph{fourth-order cumulant} is the symmetric 
	$(p\times p\times p\times p)$-tensor $\ccal_x$ such that
	\begin{equation}
	\label{def-cum4}
	\begin{aligned}
	&{ [\ccal_x^{(4)}] }{}_{i_1i_2i_3i_4} := \cum[x_{i_1}, x_{i_2}, x_{i_3}, x_{i_4}] \\
	&:= \E[ x_{i_1} x_{i_2} x_{i_3} x_{i_4}] 
	- \E[x_{i_1}x_{i_2}] \E[x_{i_3} x_{i_4}] \\
	&- \E[x_{i_1}x_{i_3}]\E[x_{i_2}x_{i_4}]
	- \E[x_{i_1}x_{i_4}]\E[x_{i_2}x_{i_4}].
	\end{aligned}
	\end{equation}
	If $x$ is not zero-mean, this definition is instead
	applied to the variable $\tilde{x} := x - \E(x)$.
	
	\subsubsection{Kurtosis}
	\label{app-def-kurt}
	
	The \emph{kurtosis} of a univariate zero-mean random variable $\alpha$
	is the number $\kappa_{\alpha}$ such that
	\begin{equation}
	\label{def-kurtosis}
	\kappa_{\alpha} := \cum[\alpha,\alpha,\alpha,\alpha]
	= \E[\alpha^4] - 3\E[\alpha^2]\E[\alpha^2].
	\end{equation}
	Note that:
	\begin{itemize}
		\item[-]
		If $\alpha$ is from the standard normal distribution 
		then $\kappa_{\alpha} = 0$;
		
		\item[-]
		If $\alpha$ is from the uniform distribution such that
		$\E(\alpha) = 0$ and $\var(\alpha) = 1$ then $\kappa_{\alpha} = -1.2$;
		
		\item[-]
		If $\alpha$ is from the Laplace distribution such that
		$\E(\alpha) = 0$ and $\var(\alpha) = 1$ then $\kappa_{\alpha} = 3$.
	\end{itemize}

	\subsection{The Fourth-Order Cumulant of the ICA Model}
	\label{app-ica-cum4}

	In this section, we recall the form of the
	fourth-order cumulant of the ICA model
	originally utilized by \citet{LatEtAl2007}
	for the FOOBI algorithm.
	This cumulant can be used for the construction of subspace $W$
	as an alternative to the procedure presented in Section~\ref{sec-subspace}.
	
	The formal definition of the fourth-order cumulant can be found in Appendix~\ref{app-def-cum}.
	By the multi-linearity and independence properties of cummulants \citep[see, e.g.,][Chapter 5]{ComJut2010},
	the fourth-order cumulant of the ICA model~\eqref{ica} is the tensor 
	$$
	\ccal_x^{(4)} = \sum_{i=1}^k \kappa_{\alpha_i} d_i \otimes d_i \otimes d_i \otimes d_i,
	$$
	where $\otimes$ stands for the outer product and $\kappa_{\alpha_i}$ is the 
	kurtosis of the $i$-th source (see Appendix~\ref{app-def-kurt} for the definition).
	As \citet{LatEtAl2007} show, the flattening of this tensor\footnote{
		The flattening of a fourth-order tensor can be defined by analogy of the vectorization of a matrix.
		In case of the ICA model, the order of indices is indifferent due to symmetry. 
		See an example of a flattening in \citet{LatEtAl2007}. 
	}
	is a matrix $C \in \R^{p^2 \times p^2}$ such that
	\begin{equation}
	\label{flattening-cum4}
	C = (D\odot D) \diag(\kappa) (D \odot D)^{\top},
	\end{equation}
	where $\odot$ stands for the Khatri-Rao product (recall the definition is Appendix~\ref{app-khatri-rao})
	and the $i$-th element of the vector $\kappa\in\R^k$ is the kurtosis $\kappa_{\alpha_i}$ of the $i$-th
	source $\alpha_i$.
	
	This expression~\eqref{flattening-cum4}
	is the key for the subspace construction.
	To see that, let us have a look at the matrix $A \in \R^{p^2 \times k}$ such that
	\begin{equation}
	\label{def-a}
	A : = D \odot D.
	\end{equation}
	The $i$-th column of this matrix $a_i := A_{:,i} = d_i \odot d_i$
	and its matricization $A_i := \mat(a_i) = \mat( d_i \odot d_i ) = d_i d_i^{\top}$
	is equal to the $i$-th atom. Therefore, using this flattening-matricization
	trick one can easily obtain an estimate of the subspace $W$.
	Indeed, a basis of $W$ can be constructed as a matricization of a basis
	of the column space of $C$. The latter can be obtained, e.g., via the eigen decomposition
	or singular value decomposition of $C$. In particular,
	let
	$C = U \Sigma V$
	be the SVD of $C$. Then a basis of the subspace $W$ can be constructed as
	$H_i := \mat(u_i)$ for every $i\in[k]$, where $u_i$ is the $i$-th left singular vector,
	and it holds that
	\begin{equation}
	\label{subspace-from-cum4}
	W = \spn\cbra{\mat(u_1), \dots, \mat(u_k)}.
	\end{equation}
	This exact construction is used in a practical implementation
	of the FOOBI algorithm and can be used as a replacement of the Step I in 
	the OverICA Algorithm~\ref{alg-overica}.

	\section{The Semi-Definite Program}
	
	\subsection{Algorithm for the Relaxation~\eqref{def-sdp-relax} }
	\label{app-sdp-fista}

	This section applies the FISTA algorithm by~\cite{BecTeb2009}
	for finding the solution of the convex relaxation~\eqref{def-sdp-relax}.
	
	\subsubsection{FISTA}
	
	Let $\kcal$ be a set of all symmetric PSD matrices with unit trace, 
	i.e. $\kcal := \cbra{B \in \scal_p : B \succeq 0, \tr(B) = 1}$,
	where $\scal_p$ denotes the set of all symmetric matrices in $\R^{p\times p}$.
	Let
	$g(B) := \iota_{\kcal}(B)$ be the indicator function of the set~$\kcal$ and
	let the negative objective of the problem~\eqref{def-sdp-relax} be
	$$
	f(B) := - \inner{G,B} + \frac{\mu}{2} \sum_{j\in[m-k+t]} \inner{B,F_j}^2.
	$$
	We then can solve the problem 
	$$\min_{B \in \R^{p\times p}} f(B) + g(B)$$ 
	with FISTA~\cite{BecTeb2009}. 
	The gradient of the differentiable part of the objective is
	$$
	\nabla f(B) = - G^{\top} + \mu \sum_{j\in[m-k+t]} \tr(F_jB)F_j^{\top},
	$$ 
	where  we used the fact that its Lipschitz constant is $L=\mu$. 
	This is summarized in Algorithm~\ref{alg-fista}.  The 
	projection on $\kcal$ can be computed roughly in $O(p^3)$ time
	(see Section~\ref{app-projection-k}).
	\begin{algorithm}[!ht]
		\caption{FISTA for~\eqref{def-sdp-relax}}
		\label{alg-fista}
		\begin{algorithmic}[1]
			\STATE Input: $Y^{(1)} = B^{(0)} \in\scal_p$, $z_1=1$
			\WHILE {not converged or $n>n_{max}$}
			\STATE $B^{(n)} = \proj_{\kcal}\sbra{Y^{(n)}-\frac{1}{L}\nabla f(Y^{(n)})}$
			\STATE $z_{n+1} = \frac{1}{2}\rbra{1+\sqrt{1+4z_n^2}}$
			\STATE $Y^{(n+1)} = B^{(n)} + \rbra{\frac{z_{n}-1}{z_{n+1}}}(B^{(n)}-B^{(n-1)})$
			\STATE $n\leftarrow n+1$
			\ENDWHILE
			\STATE Output: $B^{\ast} = B^{(n)}$ 
		\end{algorithmic}
	\end{algorithm}
	
	\subsubsection{Projection onto $\boldsymbol{\kcal}$}
	\label{app-projection-k}
	
	The projection onto the set
	$\kcal := \cbra{B \in \scal_p : B \succeq 0, \tr(B) = 1}$
	of a symmetric matrix $B$ can 
	be computed by first computing the eigendecomposition of this matrix
	$B = V\Lambda V^{\top}$ and then projecting its eigenvalues 
	$\lambda = \diag(\Lambda)$ onto the probability simplex $\Delta_p$. Then
	the projection is obtained as $\proj_{\kcal}(B) = V \diag[ \proj_{\Delta_p}(\lambda) ] V^{\top}$. 
	Note that the probability simplex is defined as
	$\Delta_p := \cbra{x\in\R^p : \|x\|_1 = 1, x \succcurlyeq 0}$
	and the projection onto the probability simplex can be computed 
	in linear time \citep[see, e.g.,][]{DucEtAl2008}.

	\subsubsection{Majorization-Minimization} 
	
	We observe in practice that
	our problem benefits significantly from the majorization-minimization 
	approach \citep[see, e.g.,][]{HutLan2004}, i.e. earlier stopping of the procedure in Algorithm~\ref{alg-fista}
	and restarting it again with a matrix $B^{(0)}_{next}$ obtained
	from the largest eigenvector of the matrix $B^{\ast}_{prev}$.
	We experimentally found that performing rather large number of
	majorization minimization steps (e.g., approx. 50) but with the relatively small
	maximal number of iterations $n_{max}=100$ gives the best experimental
	performance in terms of convergence speed and runtime. In particular, this is the setting we used for the 
	experiments presented in Section~\ref{sec-exps}.

	\subsubsection{The Choice of $G$} 
	\label{sec-choose-g}
	
	The choice of the matrix $G$ is an important part
	of our deflation procedure. We found experimentally
	that choosing $G$ that belongs to the subspace $W^{(t)}$
	is beneficial. The latter is also without loss of generality,
	since if $G$ has a component that does not belong to the 
	subspace, the inner product $\inner{G,B}$ of the objective
	would not be affected by that part as it would be orthogonal to $B$.
	
	In particular, let $F^{(t)} := \{f_j^{(t)}, j\in[m-k+t]\}$, where $f_j^{(t)} := \vect(F_j^{(t)})$,
	be a basis of the orthogonal complement at the $t$-th 
	deflation step.
	Let $H^{(t)} := \{h_i^{(t)}, i\in[k-t]\}$ be a basis of the orthogonal complement 
	of $F^{(t)}$, i.e. a basis of $W^{(t)}$. Let $u^{(t)}$ be the first left singular vector of $H^{(t)}$,
	then we set 
	$G^{(t)} := u^{(t)} u^{(t)\top}$.

	\subsection{Theory}
	
	We first recall the ellipsoid fitting problem
	and results related to the proof of our main
	Theorem~\ref{thm:main}.
	Note that we also prove slightly stronger 
	result for the ellipsoid fitting problem
	in Theorem~\ref{thm:ellipsoid-2}.
	
	We then proceed as follows. In Appendix~\ref{app-proof-lemma-extreme-points},
	we prove Lemma~\ref{lem-extreme-points}
	about extreme points of the constraint set of the 
	program~\eqref{def-sdp-theory}.
	Then, in Appendix~\ref{app-dual},
	we derive the dual of the program~\eqref{def-sdp-theory}
	and show its close relation to the ellipsoid fitting problem.
	Finally, in Appendix~\ref{app-proof-main},
	we prove Theorem{thm:main}
	by constructing a appropriate ellipsoid.

	\subsubsection{Ellipsoid Fitting}
	\label{app-ellipsoid-fitting}

	Ellipsoid fitting is the following elementary geometric question: given $k$ points $v_1,\ldots,v_k \in \R^p$, does there exist an ellipsoid passing exactly through them? That is, does there exist a matrix $Y \succeq 0$ with $v_i^\top Y v_i = 1$?
	
	We consider this problem in an average-case regime in which $p \to \infty$ and the $v_i$ are chosen independently from some distribution. \cite{SauEtAl2013} considers the case where the $v_i$ are standard Gaussians, and obtains the following:
	\begin{theorem}[\cite{SauEtAl2013}] Suppose $k \le p^{6/5-\varepsilon}$ for some fixed $\varepsilon > 0$. Then with very high probability\footnote{We take ``very high probability'' to indicate probability converging to $1$ at a rate faster than any inverse polynomial, as $p \to \infty$; ``high probability'' simply indicates probability converging to $1$ at any rate.} over $v_1,\ldots,v_k \in \R^p$ drawn independently from $\ncal(0,I)$, there exists an ellipsoid passing through those points. 
	\end{theorem}
	The same paper conjectures based on empirical evidence that ellipsoid fitting in this average-case model is possible when $k < p^2 / 4$, exhibiting a sharp threshold phenomenon. To our knowledge, this question remains open.
	
	Our first result is a slight generalization of the above which allows for small perturbations in the norm of the vectors. We will need this result later for the SDP analysis for ICA.
	\begin{theorem}
		\label{thm:ellipsoid-2}
		Let $w_1,\ldots,w_k \in \R^p$ be drawn independently from $\ncal(0,I)$. Let $v_i = \pi_i w_i$ where each $\pi_i$ is a scalar random variable satisfying the following: for any $\delta > 0$, $|1/\pi_i^2 - 1| \le p^{-1/2+\delta}$ with very high probability. (The $\pi_i$ need not be identically distributed nor independent from $w_i$ or each other.) Suppose $k \le p^{6/5-\varepsilon}$ for some fixed $\varepsilon > 0$. Then with very high probability there exists an ellipsoid passing through $v_1,\ldots,v_k$.
	\end{theorem}
	
	\noindent The rest of this section is devoted to the proof of this theorem. The proof uses many ideas from \cite{SauEtAl2013}.
	
	We will construct our ellipsoid $Y$ in the form
	$$Y = I/p + \sum_{j=1}^k \beta_j w_j w_j^\top,$$
	for some scalars $\beta_j$. We wish to satisfy constraints $v_i^\top Y v_i = 1$, i.e.\
	$$\pi_i^2 \|w_i\|^2/p + \pi_i^2 \sum_j \beta_j \langle w_i,w_j \rangle^2 = 1,$$
	which can be re-written as
	$$\sum_j \beta_j \langle w_i,w_j \rangle^2 = \frac{1}{\pi_i^2} - \frac{1}{p} \|w_i\|^2.$$
	This is a linear system $V \beta = h$ where $V_{ij} = \langle w_i,w_j \rangle^2$ and $h_i = 1/\pi_i^2 - \|w_i\|^2/p$. Therefore we take $\beta = V^{-1} h$.
	
	The linear operator $\mathcal{A}^\dagger$ that takes the vector $h$ to the matrix $\sum_{j=1}^k \beta_j w_j w_j^\top$ (with $\beta = V^{-1} h$) is studied in \cite{SauEtAl2013} (with some proofs deferred to \cite{Sau2011}), where the following bound is shown on the ``infinity-to-spectral'' norm.

	\begin{proposition}[\cite{SauEtAl2013}, Proposition~3]
		If $p = o(k)$ and $k = o(p^{4/3})$ then $$\|\mathcal{A}^\dagger\|_{\infty \to \mathrm{sp}} \le O(k^{5/4} p^{-2})$$ with very high probability over $\{w_i\}$.
	\end{proposition}
	\noindent The requirement $p = o(k)$ does not concern us because it is sufficient to prove Theorem~\ref{thm:ellipsoid-2} in the case $p = o(k)$; this is because decreasing $k$ only makes it easier to fit an ellipsoid through $k$ points.
	
	Our goal is to show $Y \succeq 0$ so it is sufficient to show $\|A^\dag h\| \le 1/p$, which we will do using $\|\mathcal{A}^\dagger h\| \le \|\mathcal{A}^\dagger\|_{\infty\to\mathrm{sp}} \|h\|_\infty$. It remains to bound $\|h\|_\infty$.

	Let $\delta > 0$, to be chosen later. Recall that $h_i = 1/\pi_i^2 - \|w_i\|^2/p$. To control the first term, we have by assumption that with very high probability, $|1/\pi_i^2 - 1| \le p^{-1/2+\delta}$ for all $i$. Note that $\|w_i\|^2 \sim \chi_p^2$. We will use the following chi-squared tail bound.
	\begin{lemma}[\cite{Sau2011}, Lemma~7]
		\label{lemma:chi-sq}
		$$\prob{|\chi_p^2 - p| \ge t} \le 2 \exp\left(-\frac{1}{8} \min\left\{\frac{t^2}{p},t\right\}\right).$$
	\end{lemma}
	\noindent This implies that $\left| \|w_i\|^2/p - 1 \right| \le p^{-1/2 + \delta}$ with very high probability. Therefore $\|h\|_\infty \le 2p^{-1/2 + \delta}$ with very high probability. To complete the proof of Theorem~\ref{thm:ellipsoid-2}, we have, using the assumption $k \le p^{6/5-\varepsilon}$,
	$$
	\begin{aligned}
	\|\mathcal{A}^\dagger h\| 
	&\le \|\mathcal{A}^\dagger\|_{\infty\to\mathrm{sp}} \|h\|_\infty 
	\le O(k^{5/4} p^{-2} \cdot p^{-1/2+\delta})  \\
	&\le O(p^{5/4(6/5-\varepsilon) - 5/2 + \delta}) = O(p^{-1 - 5\varepsilon/4 + \delta}),
	\end{aligned}
	$$
	which is less than $1/p$ for sufficiently large $p$, provided we choose $\delta$ small enough.

	\subsubsection{Proof of Lemma~\ref{lem-extreme-points}}
	\label{app-proof-lemma-extreme-points}
	
	In this section, we prove Lemma~\ref{lem-extreme-points}.
	Recall that Lemma~\ref{lem-extreme-points} states that:
	\emph{if the atoms $d_1d_1^{\top}$, $d_2d_2^{\top}$, $\dots$, $d_kd_k^{\top}$
		are linearly independent, then they are extreme points
		of the set $\kcal$ defined in~\eqref{def-kcal}.}
	
	\begin{proof}
		An extreme point of a convex set cannot be 
		expressed as a convex combination of any two 
		points from this set.
		Assume that
		an atom $d_jd_j^{\top}$, for some $j\in[k]$,
		can be expressed as
		$d_jd_j^{\top} = \lambda A + (1-\lambda) B$,
		where $A,B \in \kcal$ and $\lambda \in [0,1]$.
		Since $A,B \in W$, there exist 
		vectors $\alpha \in \R^k$ and $\beta \in \R^k$
		such that $A = \sum_{i=1}^k \alpha_i d_id_i^{\top}$
		and $B = \sum_{i=1}^k \beta_i d_id_i^{\top}$.
		Therefore, 
		$$
		d_jd_j^{\top} = \sum_{i=1}^k [\lambda \alpha_i + (1-\lambda)\beta_i]
		d_id_i^{\top}.
		$$
		This, however, contradicts to the linear independence
		of the matrices $d_1d_1^{\top}$, \dots, $d_kd_k^{\top}$.
		Note that every atom belongs to the set $\kcal$.
		Indeed, every atom $d_id_i^{\top}$ is a positive semi-definite
		matrix by definition and it has unit trace by Assumption~\ref{ass-scaling}.
	\end{proof}

	\subsubsection{The Dual}
	\label{app-dual}
	
	In this section, we derive the dual of the program~\eqref{def-sdp-theory}.
	The first constraint, $B \in \spn\cbra{d_1d_1^{\top}, d_2d_2^{\top}, \dots, d_kd_k^{\top}}$, is equivalent to 
	$B = \sum_{i=1}^k \beta_i d_id_i^{\top}$ for some $\beta \in \R^k$
	and one gets an equivalent to~\eqref{def-sdp-theory} program
	\begin{equation}
	\label{def-sdp-theory2}
	\begin{aligned}
	\beta^{\ast} := &\argmin_{\beta \in\R^k} \; -\sum_{i=1}^k \beta_i d_i^{\top}Gd_i \\
	& \sum_{i=1}^k \beta_i d_id_i^{\top} \succeq 0, \\
	& \sum_{i=1}^k \beta_i = 1, \\
	\end{aligned}
	\end{equation}
	where the last constraint is obtained from $\tr(B) = 1$
	and Assumption~\ref{ass-scaling}.
	The original variable $B^{\ast}_{sdp}$
	is then obtained as $B^{\ast}_{sdp} = \sum_{i=1}^k \beta^{\ast}_i d_id_i^{\top}$. The Lagrangian of this problem 
	is
	\begin{equation}
	\label{lagrangian}
	\begin{aligned}
	\lcal(\beta; \; \lambda, Z )& = 
	\lambda \left(\sum_{i=1}^k \beta_i - 1\right) 
	- \inner{Z, \sum_{i=1}^k \beta_i d_id_i^{\top}} 
	 - \sum_{i=1}^k \beta_i d_i d_i^{\top} \\
	&= \sum_{i=1}^k \left[ \beta_i 
	\inner{d_id_i^{\top}, \lambda I - G - Z} \right] - \lambda,
	\end{aligned}
	\end{equation}
	where $\lambda \in \R$ and $Z \succeq 0$ are the
	Lagrange dual variables.
	The Lagrangian is linear in $\beta$ and its infimum is finite 
	only if $\inner{d_id_i^{\top}, \lambda I - G - Z} = 0$
	for all $i\in[k]$. Therefore, the dual problem takes the form
	\begin{equation}
	\label{def-dual}
	\begin{aligned}
	&\maximize_{\lambda\in\R, Z \succeq 0} \; - \lambda \\
	& \lambda \norm{d_i}_2^2 - d_i^{\top} G d_i = d_i^{\top} Z d_i, \quad
	\text{for all} \quad i\in[k].
	\end{aligned}
	\end{equation}
	
	The following lemma is then follows.
	\begin{lemma}
		\label{lemma-ellipsoid}
		An atom $d_jd_j^{\top}$ for some $j\in[k]$ is the optimizer 
		of the program~\eqref{def-sdp-theory} if and only if
		there exists a $Z\succeq0$ such that 
		\begin{equation}
		\label{cond-exist}
		d_i^{\top} Z d_i = d_j^{\top} G d_j \norm{d_i}_2^2 - d_i^{\top} G d_i,
		\quad i\in[k], i\ne j.
		\end{equation}
	\end{lemma}
	
	\begin{proof}
		One of the atoms $d_jd_j^{\top}$ is an optimizer of the 
		primal problem~\eqref{def-sdp-theory} if and only if
		the optimizer of the equivalent program~\eqref{def-sdp-theory2}
		is a $\beta$ such that $\beta_j = 1$ and $\beta_i = 0$ for all
		$i \in[k]$ and $i\ne j$. Let the dual problem~\eqref{def-dual}
		be feasible. Then, since the relative interior of the 
		program~\eqref{def-sdp-theory2} is non-empty,
		the strong duality holds. Therefore, the optimal value of $\lambda$
		must be $\lambda = d_j^{\top} G d_j$ and then the dual is feasible if
		and only if
		$d_i^{\top} Z d_i = d_j^{\top} G d_j \norm{d_i}_2^2 - d_i^{\top} G d_i$
		for every $i\in[k]$ and $i\ne j$.
	\end{proof}
	In Appendix~\ref{app-proof-main},
	we construct such an ellipsoid in order to prove 
	our main identifiability result.

	\subsubsection{Proof of Main Theorem~\ref{thm:main}}
	\label{app-proof-main}
	
	In this section, we prove our main theoretical result 
	stated in Theorem~\ref{thm:main}
	that provides identifiability results for the program~\eqref{def-sdp-theory}.
	For convenience, we recall the problem formulation.

	Let the vectors $d_1,\ldots,d_k$ be drawn \iid from the unit sphere in $\R^p$. We wish to recover the atoms $d_i d_i^\top$ from the subspace $\mathrm{span}\{ d_i d_i^\top \} \subset \R^{p \times p}$. To this end, we consider the following SDP:
	\begin{program}\label{prog:sdp}
		\begin{align*}
		\text{maximize } & \langle G,B \rangle \\
		\text{subject to } & B \succeq 0, \\
		& \tr (B) = 1, \\
		& B \in \mathrm{span}\{ d_i d_i^\top \}.
		\end{align*}
		Here $G \in \R^{p \times p}$ is some objective matrix, to be chosen randomly from some ensemble.
	\end{program}
	We are interested in understanding the performance of this SDP, when the objective $G$ is chosen as a random rank-one matrix, i.e.\ as $u u^\top$ for a vector $u$ drawn uniformly from the unit sphere in $\R^p$ (independently from $\{d_i\}$). Our main result is the following:
	
	\begin{theorem-nonumber}[Theorem~\ref{thm:main}]
		Let $\varepsilon > 0$. Consider a regime with $p$ tending to to infinity, and with $k$ varying according to the bound $k < (2 - \varepsilon) p \log p$. As above, let the $d_i$ be random unit vectors and let $G = uu^\top$ for a random unit vector $u$. Then with high probability\footnote{Throughout, ``with high probability'' indicates probability tending to $1$ as $p \to \infty$.}, the matrix $d_i d_i^\top$ for which $d_i^\top G d_i$ is largest is the unique maximizer of Program~\ref{prog:sdp}.
	\end{theorem-nonumber}
	
	\noindent The rest of this section is devoted to proving this theorem. Throughout the proof, it will be convenient to consider the following equivalent formulation of Program~\ref{prog:sdp}.
	
	\begin{program}\label{prog:alpha}
		\begin{align*}
		\text{maximize } & \sum_{i=1}^k c_i \alpha_i \\
		\text{subject to } & B \defeq (1 + \alpha_1) d_1 d_1^\top + \sum_{i > 1} \alpha_i d_i d_i^\top \succeq 0,\\
		& \sum_{i=1}^k \alpha_i = 0,
		\end{align*}
		where $c_i = d_i^\top G d_i = \langle u, d_i \rangle^2$ and we have re-indexed the $d_i$ such that $d_i^\top G d_i$ are in decreasing order (so that $d_1 d_1^\top$ is the matrix we hope to recover). Our goal is to prove that $\alpha = 0$ is the unique optimal solution to Program~\ref{prog:alpha}. (The objective values of Programs~\ref{prog:sdp} and \ref{prog:alpha} differ by an additive constant but this has no effect on the argmax.)
	\end{program}

	First sample the random vector $u$. Since the norm of $u$ does not affect the argmax of the SDP, we can take (for convenience) $u$ to be a uniformly random vector of norm $\sqrt{p}$. By a change of basis we can assume without loss of generality that $u = \sqrt{p} e_1$ where $e_1$ is the first standard basis vector.
	
	Next sample the first coordinate $(d_i)_1$ of each $d_i$ and let $c_i = \langle u,d_i \rangle^2 = p (d_i)_1^2$. Re-index so that the $c_i$ are in decreasing order. A typical $c = (c_1,\ldots,c_k)$ has the following properties.
	
	\begin{lemma}
		\label{lemma:c}
		Let $\eta > 0$. With high probability, $c$ satisfies
		\begin{enumerate}
			\item \quad $(2-\eta) \log k \le c_2 \le c_1 \le (2+\eta) \log k$,
			\item \quad $1-\eta \le \frac{1}{k} \sum_i c_i \le 1+\eta$,
			\item \quad $c_k \ge \frac{1}{k \log k}$, and
			\item \quad $c_1 - c_2 \ge \frac{1}{k^2 \log k}$.
		\end{enumerate}
		Recall that we have indexed so that $c_1 \ge c_2 \ge \cdots \ge c_k$.
	\end{lemma}
	\begin{proof}
		The $c_i$ are independent and each is distributed as $p\, g_1^2/(\sum_{j=1}^p g_j^2)$ where $g_j \sim \ncal(0,1/p)$. By the Chernoff bound we have for any $\eta > 0$,
		$$
		\begin{aligned}
		\problr{\sum_{j=1}^p g_j^2 \le 1-\eta} \le ((1-\eta)e^{\eta})^{p/2}, \\ 
		\problr{\sum_{j=1}^p g_j^2 \ge 1+\eta} \le ((1+\eta)e^{-\eta})^{p/2}.
		\end{aligned}
		$$ 
		By a union bound over the $k$ indices we have with high probability that for every $c_i$,
		$1-\eta \le \sum_{j=1}^p g_j^2 \le 1+\eta$. It is therefore sufficient to prove (i), (ii), (iii) in the case where the $c_i$ are \iid distributed as $p \, g_1^2 \sim \chi_1^2$. (i) follows from well-known results on order statistics of \iid Gaussians (see e.g.\ \cite{Bov2005}). (ii) follows by the Chernoff bound. To prove (iii), note that since the $\chi_1^2$ PDF is bounded above by a constant $C$, we have $\prob{\chi_1^2 \le r} \le C r$; now take a union bound over all $k$ and set $r = 1/(k \log k)$.
		
		To prove (iv) we will prove the stronger statement that no two entries of $c_i$ are within distance $1/(k^2 \log k)$ of each other. Fix a pair $i,j$ with $i \ne j$ and fix any value for $c_i$. The PDF of the distribution of $c_j$ is bounded above by a constant $C$ (uniformly over all $p$),
		so we have $\prob{|c_i - c_j| \le r} \le 2 C r$ over the randomness of $c_j$. The proof now follows by setting $r = 1/(k^2 \log k)$ and taking a union bound over all ${k \choose 2}$ pairs of indices.
	\end{proof}

	From this point onward we will fix a vector $c$ satisfying the conclusion of Lemma~\ref{lemma:c} (for some $\eta$ to be chosen later). Let $\bar d_i$ denote the component of $d_i$ orthogonal to $e_1$ so that $d_i = (d_i)_1 e_1 + \bar d_i$. Note that once $c$ is fixed, $\bar d_i$ is a uniformly random vector on the sphere of radius $\sqrt{1 - (d_i)_1^2}$. The following key lemma shows how to prove various inequalities on $\alpha$ that are valid for any feasible solution to Program~\ref{prog:alpha}.
	
	\begin{lemma}\label{lemma:ineq}
		Let $\gamma > 0$. Fix $c$ satisfying the conclusion of Lemma~\ref{lemma:c} with some parameter $\eta$. For any set $S \subseteq [k]$ of size at most $(1-\gamma)p$, with $1 \in S$, it holds with very high probability (over the randomness of $\{\bar d_i\}$) that every feasible point for Program~\ref{prog:alpha} satisfies $\sum_{i \in S} \alpha_i \leq 0$.
	\end{lemma}
	\begin{proof}
		Let $\mathcal{D} = \mathrm{span}(\{d_i \mid i \in S\} \cup \{e_1\})$, and for $i \not\in S$, let $v_i = P_{\mathcal{D}^\perp} d_i$, the orthogonal projection onto $\mathcal{D}^\perp$. We will show how to use Theorem~\ref{thm:ellipsoid-2} to construct an ellipsoid $Y$ on the subspace $\mathcal{D}^\perp$ that passes through the vectors $\{ v_i \mid i \not\in S \}$; we extend the quadratic form $Y$ to the entire space $\R^p$ by acting as $0$ on $\mathcal{D}$. Then since $Y \succeq 0$, we have for any feasible point $B$ of the Program~\ref{prog:alpha}:
		$$ 
		\begin{aligned}
		0 \leq \langle Y,B \rangle 
		&= \sum_{i \not\in S} \alpha_i d_i^\top Y d_i 
		= \sum_{i \not\in S} \alpha_i v_i^\top Y v_i \\
		&= \sum_{i \not\in S} \alpha_i = -\sum_{i \in S} \alpha_i ,
		\end{aligned}
		$$
		\noindent which yields the desired inequality.
		
		It remains to show that (with very high probability) we can construct the ellipsoid $Y$. Choose an orthonormal basis so that the first coordinate is still $e_1$ (parallel to $u$), the first $|S| + 1$ coordinates span $\mathcal{D}$, and the remaining $p' = p - |S| - 1$ coordinates span $\mathcal{D}^\perp$. In this basis, write $d_i = [(d_i)_1 \; x_i^\top \; v_i^\top]^\top$ with $x_i \in \R^{|S|}$ and $v_i \in \R^{p'}$. With $\tilde x_i \sim \ncal(0,I_{|S|}/p)$ and (independently) $\tilde v_i \sim \ncal(0,I_{p'}/p)$ we have $v_i = \tilde v_i/\sqrt{(d_i)_1^2 + \|\tilde x_i\|^2 + \|\tilde v_i\|^2} = \pi_i \tilde v_i$ where $\pi_i = \left((d_i)_1^2 + \|\tilde x_i\|^2 + \|\tilde v_i\|^2\right)^{-1/2}$. In order to invoke Theorem~\ref{thm:ellipsoid-2} and complete the proof, we need to show that for any $\delta > 0$, $|1/\pi_i^2-1| \le (p')^{-1/2+\delta}$ with very high probability. We have $1/\pi_i^2 = (d_i)_1^2 + \|\tilde x_i\|^2 + \|\tilde v_i\|^2 \sim (d_i)_1^2 + \frac{1}{p}\chi_{p-1}^2$. The result now follows by combining the facts $p' \ge \frac{1}{2} \gamma \,p$ and $(d_i)_1^2 = c_1/p \le (2+\eta) (\log k) / p$ with the chi-squared tail bound (Lemma~\ref{lemma:chi-sq}).
	\end{proof}

	We will choose a collection $\mathcal{S}$ (depending on $c$ but not $\{\bar d_i\}$) of sets $S$ to which we will apply Lemma~\ref{lemma:ineq}. The idea will be to combine the constraints from Lemma~\ref{lemma:ineq} to produce the constraint $\sum_{i=1}^k c_i \alpha_i \le 0$, showing that $\alpha = 0$ is an optimal solution to Program~\ref{prog:alpha}. (We will later argue why it is the unique optimum.)
	
	We can construct a random $S \subseteq [k]$ by including each $i \ge 2$ independently with probability $q_i = c_i/c_2$ (and always including index 1). (Recall that we have indexed so that the $c_i$ are decreasing.) Let the collection $\mathcal{S}$ consist of $N = k^{11}$ subsets constructed independently by the above process.
	
	In order to apply Lemma~\ref{lemma:ineq}, we need to check that each $S \in \mathcal{S}$ has size at most $(1-\gamma)p$.
	\begin{lemma}
		\label{lemma:S-size}
		Suppose $k \le p^{6/5-\varepsilon}$ for some fixed $\varepsilon > 0$. There exist $\eta > 0$ and $\gamma > 0$ (both depending on $\varepsilon$) so that the following holds. Fix $c$ satisfying the conclusion of Lemma~\ref{lemma:c} with parameter $\eta$. With high probability, every $S \in \mathcal{S}$ satisfies $|S| \le (1-\gamma)p$.
	\end{lemma}
	\begin{proof}
		For each $S \in \mathcal{S}$ we have $\EE|S| = 1 + \sum_{i > 1} \frac{c_i}{c_2} \le \sum_{i \ge 1} \frac{c_i}{c_2} \le \frac{(1+\eta)k}{c_2} \le \frac{(1+\eta)k}{(2-\eta)\log k}$ using Lemma~\ref{lemma:c}. By Hoeffding's inequality, for any $t \ge 0$, $\prob{|S| - \EE|S| \ge tk} \le \exp(-2 k t^2)$. Letting $t = 1/\log^2 k$ and taking a union bound over all $S \in \mathcal{S}$ we have that with high probability, every $S \in \mathcal{S}$ satisfies $|S| \le \frac{(1+\eta)k}{(2-\eta)\log k} + \frac{k}{\log^2 k} = (1+o(1))\frac{(1+\eta)k}{(2-\eta)\log k}$. Using the hypothesis $k \le (2-\varepsilon) p \log p$ and taking $\eta$, $\gamma$ small enough, this yields $|S| \le (1+o(1))\frac{(1+\eta)(2-\varepsilon)p \log p}{(2-\eta)\log((2-\varepsilon)p \log p)} \le (1-\gamma)p$ for sufficiently large $p$.
	\end{proof}
	
	Let $n_i$ denote the number of sets $S \in \mathcal{S}$ in which $i$ appears. The following lemma shows concentration of the $n_i$.
	\begin{lemma}
		\label{lemma:del}
		Let $\delta = k^{-4}$. Fix $c$ satisfying the conclusion of Lemma~\ref{lemma:c} with some parameter $\eta > 0$. With high probability, for all $i$ we have $(1-\delta)q_i \le \frac{n_i}{N} \le (1+\delta)q_i$.
	\end{lemma}
	\begin{proof}
		Note that $n_i \sim \Binom(N,q_i)$. By Hoeffding's inequality, $\prob{n_i \le N(1-\delta)q_i} \le \exp(-2 N \delta^2 q_i^2) \le \exp(-2 k^{11} k^{-8} (c_k/c_2)^2) \le \exp(-2 k^3 / ((2+\eta)k \log^2 k)^2) = \exp(-k/\mathrm{polylog}(k))$, using Lemma~\ref{lemma:c}. The same bound also holds for $\prob{n_i \ge N(1+\delta)q_i}$. Taking a union bound over all $k$ indices $i$, we obtain the desired result.
	\end{proof}

	Let $\mathcal{\hat S}$ be the collection consisting of all sets in
	$\mathcal{S}$ along with the additional sets $\{1\}$ and $\{1,i\}$ for each
	$i \ge 2$. Since there are polynomially-many sets in $\mathcal{\hat S}$ 
	we have (by Lemma~\ref{lemma:ineq} and a union bound) that with high
	probability, every feasible $\alpha$ for Program~\ref{prog:alpha} satisfies
	$\sum_{i \in S} \alpha_i \le 0$ for every $S \in \mathcal{\hat S}$. Our next
	step is to combine these constraints to make the constraint
	$\sum_{i=1}^k c_i \alpha_i \le 0$. In other words, we need to form the
	vector $c$ as a conic combination of the vectors 
	$\{\one_S \mid S \in \mathcal{\hat S}\}$. We can do this as follows:
	
	$$c = \frac{c_2}{(1+\delta)N} \sum_{S \in \mathcal{S}} \one_S + \sum_{i=2}^k A_i \one_{\{1,i\}} + b\, \one_{\{1\}},$$
	where
	$$A_i = c_i - \frac{c_2 n_i}{(1+\delta)N}$$
	and
	$$b = c_1 - \frac{c_2}{1+\delta} - \sum_{i > 1} A_i.$$
	
	\noindent The first term is a uniform combination of the constraints from $\mathcal{S}$; by the construction of $\mathcal{S}$, this is already close to $c$. The remaining two terms correct for the discrepancy.
	
	It remains to check $A_i \ge 0$ and $b \ge 0$. Lemma~\ref{lemma:del} implies that $0 \le A_i \le \frac{2\delta c_i}{1+\delta}$. Using Lemma~\ref{lemma:c} we have
	
	$$b \ge (c_1 - c_2) - \sum_{i > 1} \frac{2 \delta c_i}{1+\delta} \ge \frac{1}{k^2 \log k} - 2\delta(1+\eta)k > 0$$
	by the choice of $\delta = k^{-4}$. This completes the proof that $\alpha = 0$ is an optimal solution to Program~\ref{prog:alpha}.
	
	To complete the proof of Theorem~\ref{thm:main} we need to show that $\alpha = 0$ is the unique optimum. Let $\hat c_1 = c_1 - \xi$ for an arbitrary small constant $\xi > 0$, and $\hat c_i = c_i$ for $i \ge 2$. Let $P_1$ denote Program~\ref{prog:alpha} and let $P_2$ denote Program~\ref{prog:alpha} with the objective changed from $c$ to $\hat c$. The above argument shows that (provided $\xi$ is small enough) $\alpha = 0$ is an optimal solution to $P_2$ (as well as $P_1$); to see this, note that we can form $\hat c$ as a conic combination of constraints simply by decreasing $b$ by $\xi$. This means that any optimal solution $\alpha^*$ to $P_1$ must have $\alpha^*_1 \ge 0$, or else it would outperform the zero solution in $P_2$. But we have the constraint $\alpha^*_1 \le 0$ (taking $S = \{1\}$) and so $\alpha^*_1 = 0$. We also have $\alpha^*_1 + \alpha^*_i \le 0$ (taking $S = \{1,i\}$) and $\sum_i \alpha^*_i = 0$, which together imply $\alpha^* = 0$.

	\section{Experiments}
	\label{app-exps}
	
	In this section, we describe the synthetic data 
	and error metrics
	used for the experiments in Section~\ref{sec-exps}.

	\subsection{Sampling Procedures}
	\label{app-sampling}
	
	In this appendix, we describe in details all the sampling procedures
	that were used for the experiments in Section~\ref{sec-exps}.
	
	\subsubsection{Sampling Mixing Matrix}
	\label{app-samling-mm}
	
	Given a fixed pair $(p,k)$, we sample a mixing matrix
	$D\in\R^{p \times k}$ as follows. For every column:
	\begin{itemize}
		\item[1)] Sample a $p$-valued vector $d_i$ from the standard normal distribution;
		\item[2)] Normalize to unit norm: $d_i \leftarrow d_i / \norm{d_i}_2$.
	\end{itemize}
	This is the default sampling procedure for any mixing matrix
	in this paper.
	
	It is also interesting to consider two modifications of this sampling procedure:
	(a) sampling with pruning and (b) sampling with sparseness.
	
	In the former (prune) case, we reject the sampled matrix
	if its coherence $\sigma(D)$ defined in equation~\eqref{def-coherence}
	exceeds the threshold $\wb{\sigma}$. We use the following definition
	of coherence
	\begin{equation}
	\label{def-coherence}
	\sigma(D) := \max_{i\ne j} \; \abs{ \inner{d_i,d_j} },
	\end{equation}
	where $\norm{d_i}_2 = 1$
	for all $i\in[k]$ \citep[see, e.g.,][]{AnaEtAl2015}.
	This coherence $\sigma(D)$ takes values in $[0,1]$
	and is the cosine of the angle between two mixing component
	with the smallest mutual angle. It is intuitively clear that
	it is more difficult to recover latent mixing components
	with smaller angle between them. In practice,
	we set up the threshold $\wb{\sigma}$ to the mean value of the
	coherence for a given pair $(p,k)$ over large number of resampling
	(say $10,000$).
	
	In the latter (sparse) case, we first sample a matrix
	from the normal distribution as described above and
	then zero-out half of the elements of this matrix.
	To construct the support, we sample another matrix from the normal distribution and zero-out all the elements exceeding the median value.
	If the obtained matrix has at least one column of all zeros or 
	the respective atoms $d_id_i^{\top}$ are not linearly independent,
	we resample such matrix.
	
	\begin{figure*}[t]
		\centering
		\begin{tabular}{c c}
			\includegraphics[width=0.5\textwidth]{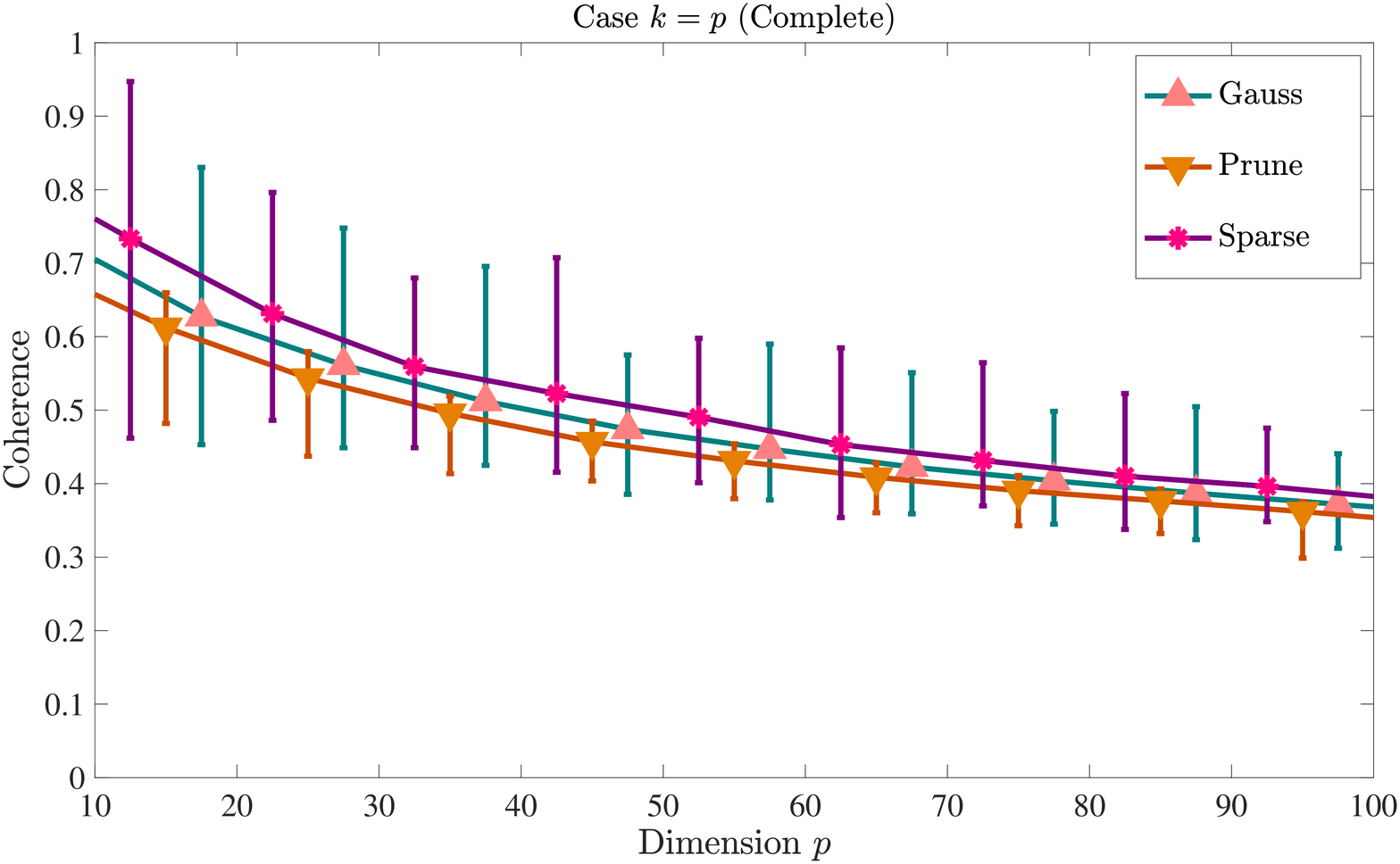}
			&
			\includegraphics[width=0.5\textwidth]{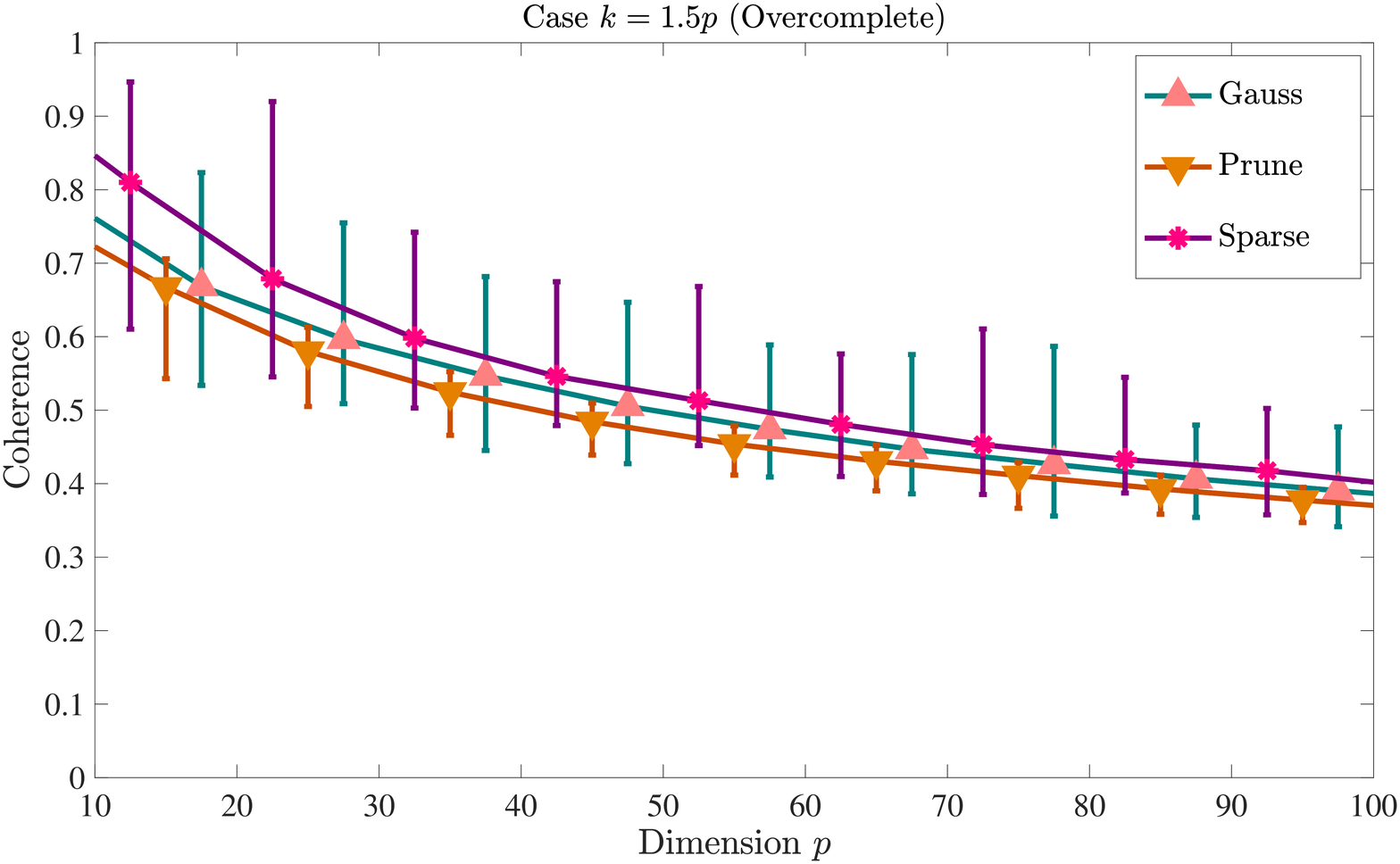}
		\end{tabular}
		\caption{
			Coherence of mixing matrices
			sampled with the sampling procedures described in Appendix~\ref{app-samling-mm}
			in the complete (Left)
			and overcomplete (Right) cases.
			The lines correspond to median values over $1.000$
			samples and the ticks, respectively, to the minimum and maximum
			values.
		}
		\label{fig-coh}
	\end{figure*}
	Importantly, the obtained matrices in these three cases
	are not that much different for the purposes of the overcomplete 
	recovery. Indeed, we can see from the a simple simulation
	experiment that they always have high coherences
	(see Figure~\ref{fig-coh}). 
	The coherence is especially high for lower dimensions $p$,
	such as 10 or 20, which are more amenable to experimental
	comparison. Therefore, one has to be careful when interpreting
	the results in such cases.
	
	As we have seen in extensive experimental comparison,
	the OverICA algorithm recovers equally well mixing matrices
	sampled from any of these three sampling type.

	\subsubsection{The Population Case}
	\label{app-sampling-population}
	
	This synthetic data simulate the 
	infinite sample
	scenario. The algorithms are then provided
	with the exact subspace and one can measure how well
	the algorithms estimate the subspace
	in this noiseless setting. 
	This type of synthetic data can only be used
	either with FOOBI or OverICA algorithms.

	This data is for the scenario where the dimension $p$
	of observations is fixed and the latent dimension $k$
	is changing from $p$ up to $p(p+1)/2$. In particular,
	we sample $N_{rep}$ instances of
	synthetic data for different pairs of $(p,k)$:
	\begin{enumerate}
		\item Fix the dimension of observations $p$;
		\item Repeat $N_{rep}=5$ times for different values of $k$:
		\begin{itemize}
			\item[-] Given a pair $(p,k)$, sample a mixing matrix~$D$
			(see Appendix~\ref{app-samling-mm});
			\item[-] Construct the matrix $A:= D\odot D$;\footnote{
				See Appendix~\ref{app-ica-cum4} for the explanation
				how exactly this matrix is related to the subspace $W$.
			}
			\item[-] Construct the matrix $C := AA^{\top}$;\footnote{
				Since we are only interested in the construction
				of a basis of the matrix $C$ from equation~\eqref{flattening-cum4},
				we can omit the scaling, i.e. $\diag(\kappa)$.
			}
			\item[-] Extract an orthonormal basis $H:=\sbra{h_i,i\in[k]}$ of the column space
			of the matrix $A$ from the matrix $C$
			(in practice, as the largest $k$ eigenvectors or singular vectors of $C$);
			\item[-] 
			Use the matrices $\cbra{H_1, \dots, H_k}$
			where each $H_i := \mat(h_i)$ for all $i\in[k]$, 
			as a basis of the subspace as an input to the second
			step of the OverICA or FOOBI algorithms.
		\end{itemize}
	\end{enumerate}
	
	\subsubsection{The Finite Sample Case: Fixed Dimensions}
	\label{app-sampling-finite}
	
	In this scenario, a finite sample 
	$X:=\cbra{x^{(1)}, \dots, x^{(n)}}$ is sampled exactly
	from the ICA model~\eqref{ica}. This imitates presence 
	of the finite sample noise but absence of 
	the model misspecification.
	The amount of noise can be controlled by the
	number of samples $n$. 
	This data can be used as an input to any overcomplete ICA algorithm.
	
	In particular, we fix both dimensions $p$ and $k$
	and vary the number of samples $n$:
	\begin{enumerate}
		\item Fix the dimension of the observations $p$ and the
		latent dimension $k$;
		\item Repeat ($n_{rep}=10$ times) for different sample sizes $n$:
		\begin{itemize}
			\item[-] Sample a mixing matrix $D$ (see Appendix~\ref{app-samling-mm});
			\item[-] Sample $n$ observations from the ICA model~\eqref{ica}
			with the uniformly distributed on the interval $[-0.5; 0.5]$ sources.
			\item[-] Use this sample $X:=\cbra{x^{(1)}, \dots, x^{(n)}}$ as an input.
		\end{itemize}
	\end{enumerate}
	This procedure results in $n_{rep}$ datasets for any $n$
	for any pair of $(p,k)$.

	\subsubsection{The Finite Sample Case: Fixed Sample Size}
	\label{app-sampling-fsfixedn}
	
	This sampling procedure is almost identical to the one described in Section~\ref{app-sampling-finite}
	with the only difference that the sample size $n$ is fixed instead and the latent dimension $k$ is varied:
	\begin{enumerate}
		\item Fix the dimension of the observations $p$ and the sample size $n$;
		\item Repeat ($N_{rep}=10$ times) for different latent dimensions $k$:
		\begin{itemize}
			\item[-] Sample a mixing matrix $D$ (see Appendix~\ref{app-samling-mm});
			\item[-] Sample $n$ observations from the ICA model~\eqref{ica}
			with the uniformly distributed on the interval $[-0.5; 0.5]$ sources.
			\item[-] Use this sample $X:=\cbra{x^{(1)}, \dots, x^{(n)}}$ as an input.
		\end{itemize}
	\end{enumerate}
	
	It does not make sense to consider values of $k$ greater than $p(p+1)/2$, 
	since in that case the matrix $A:= D\odot D$ does not have full column rank.
	In practice, we point out some interesting values of $k$: (a) $k=p$, (b) $k=p^2/4$
	(corresponds to phase transition of OverICA), (c) $k=p(p-1)/2$ and $k=p(p+1)/2$.
	We mark these values with vertical green lines on plots for all experiments which use 
	this sampling procedure.

	\subsection{Error Metrics}
	\label{app-errors}
	
	Given a ground truth mixing matrix $D$
	and its estimate $\wh{D}$, we introduce
	the following error metrics to measure 
	the estimation quality.
	Note that for the computation of these error metrics
	we assume that every mixing 
	component, i.e. every column of the mixing matrix,
	have unit norm in accordance with Assumption~\ref{ass-scaling}.
	
	\subsubsection{F-Error}
	\label{app-f-error}
	
	We define the \emph{f-error}, i.e. the Frobenius error, as:
	\begin{equation*}
	\mathrm{err}_F(D, \wh{D}) :=
	\min_{\sigma \in \pcal} \frac{ \norm{D - \wh{D}_{\sigma}}_F^2 }{ \norm{D}_F^2 },
	\end{equation*}
	where $\norm{\cdot}_F$ stands for the Frobenius norm of 
	a matrix and we minimize the error over all possible
	permutations $\sigma\in\pcal$ of the columns of $\wh{D}$
	(with the Hungarian algorithm in practice \citep{Kuh1995}).
	Smaller values of this error are better.
	
	\subsubsection{A-Error}
	\label{app-a-error}
	
	We define the \emph{a-error}, i.e. the angle error, as:
	\begin{equation*}
	\begin{aligned}
	\mathrm{err}_C(D, \wh{D}) &:=
	\frac{2}{k\pi} \; \min_{\sigma\in\pcal} \; 
	\sbra{ 
		\sum_{i\in [k]} \acos \rbra{ 
			\gamma  } 
	}, \\
	\gamma &:= \frac{\abs{ \inner{d_i, \wh{d}_{\sigma(i)}} }   }{\norm{d_i}_2 \norm{\wh{d}_{\sigma(i)}}_2},
	\end{aligned}
	\end{equation*}
	where $\norm{\cdot}_2$ stands for the Euclidean norm of
	a vector, the $d_i$ or $\wh{d}_i$ are the $i$-th columns
	of the matrices $D$ or $\wh{D}$, respectively, 
	and we again minimize the error over all
	possible permutations of the columns of $\wh{D}$.
	Note that $\pi\approx 3.14$. The a-error takes values in the interval $[0,1]$
	and smaller values of the a-error are better.
	Note the relation of the a-error to the coherence
	measure~\eqref{def-coherence}.

	\subsubsection{Number of Recovered Atoms}
	\label{app-error-recov}
	
	Since neither a- nor f-errors measure the quality
	of recovery of individual mixing components,
	we also introduce another metric for the estimation
	recovery, which measures the number of ``perfectly''
	recovered components.
	
	\paragraph{Perfect Recovery.}
	
	By a ``perfectly'' recovered component we mean
	a component $\wh{d}_{\sigma(i)}$ which is at
	most angle $\theta$ far from its respective
	ground truth value $d_i$, i.e.
	$$
	\acos \rbra{ \gamma } =
	\acos \rbra{ 
		\frac{\abs{ \inner{d_i, \wh{d}_{\sigma(i)}} }   }{\norm{d_i}_2 \norm{\wh{d}_{\sigma(i)}}_2  } 
	}  \le \theta,
	$$
	where $\sigma$ corresponds to the optimal permutation
	in terms of a-error as described above.
	
	\paragraph{Normalized Recovery Vector.}
	
	We define the \emph{normalized recovery vector} $r\in[0,1]^k$
	such that its $i$-th component is equal to the fraction
	of at least $i$ ``perfectly'' recovered (in terms of the parameter $\theta$)
	components over $N_{rep}$ repetitions of an experiment.
	For example, if $k=5$ and $N_{rep}=3$ and an algorithm recovers
	``perfectly'' 2, 4, and 3 components in these 3 runs,
	then the normalized recovery vector is $r = (1,1, 2/3, 1/3,0)$.
	In the plots in Section~\ref{sec-exps}, we use \textbf{black}
	for 100\% recovery of at least $i\le k$ components, i.e. 1's, and \textbf{white} for 
	never recovering $i\le k$ components or more ``perfectly,'' i.e. 0's. The
	intermediate values are shown in \textbf{grey}.
	We consider the threshold value of 
	$\theta:= \acos(0.99)$,
	which corresponds to the angle $\phi := \theta*180/\pi \approx 8$.

\begin{figure*}[!t]
	\centering
	\begingroup
	\setlength{\tabcolsep}{0pt} 
	\renewcommand{\arraystretch}{.5} 
	\begin{tabular}{c@{\hskip 2mm}  c@{\hskip 2mm}  c}
		\includegraphics[width=.32\textwidth]{f3-ferr_smalln.eps}
		& 
		\includegraphics[width=.32\textwidth]{f3-aerr_smalln.eps}
		&
		\includegraphics[width=.32\textwidth]{f3-rec_oica.eps}
		\\
		\includegraphics[width=.32\textwidth]{f3-ferr.eps}
		&
		\includegraphics[width=.32\textwidth]{f3-aerr.eps}
		&
		\includegraphics[width=.32\textwidth]{f3-rec_foobi.eps}
		\\
		\includegraphics[width=.32\textwidth]{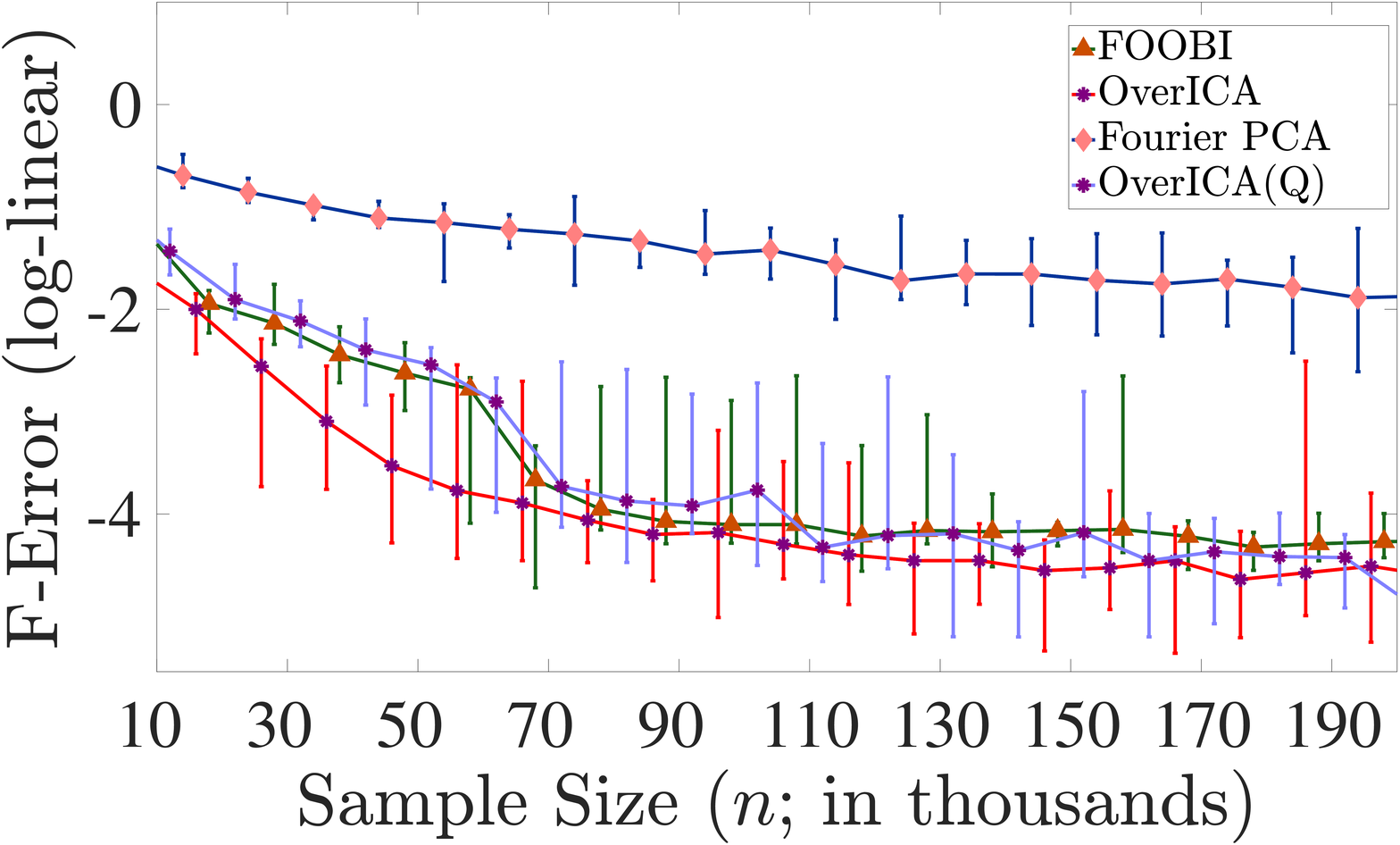}
		&
		\includegraphics[width=.32\textwidth]{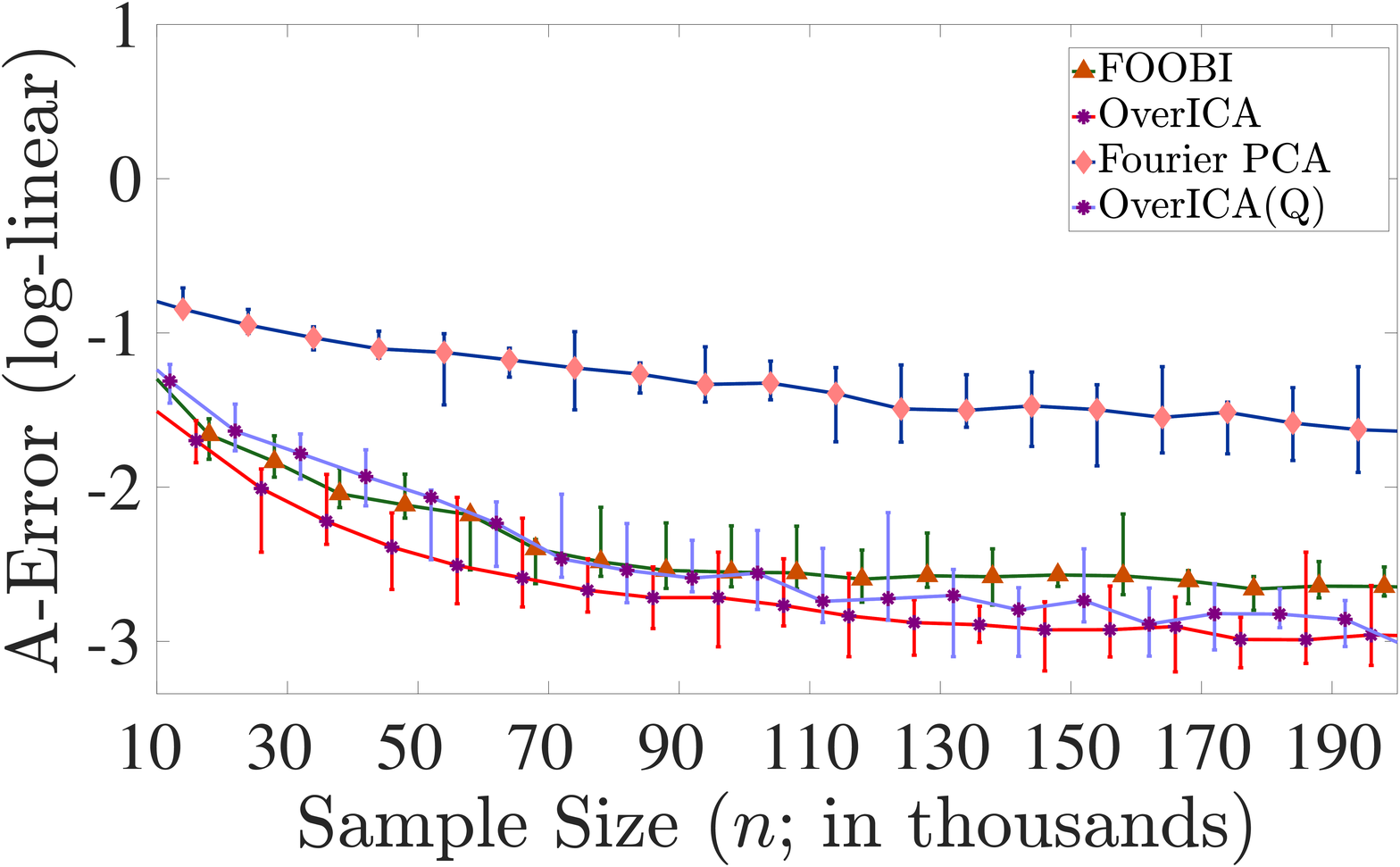}
		&
		\includegraphics[width=.32\textwidth]{f3-rec_oicaq.eps}
	\end{tabular}
	\endgroup
	\caption{
		Comparison in the finite sample regime -- additional plots.
		See explanation in Section~\ref{sec-exps-fs}.
	}
	\label{fig-fs-more}
\end{figure*}

\begin{figure*}[t]
	\centering
	\begingroup
	\setlength{\tabcolsep}{0pt} 
	\renewcommand{\arraystretch}{.5} 
	\begin{tabular}{c c c}
		\includegraphics[width=.33\textwidth]{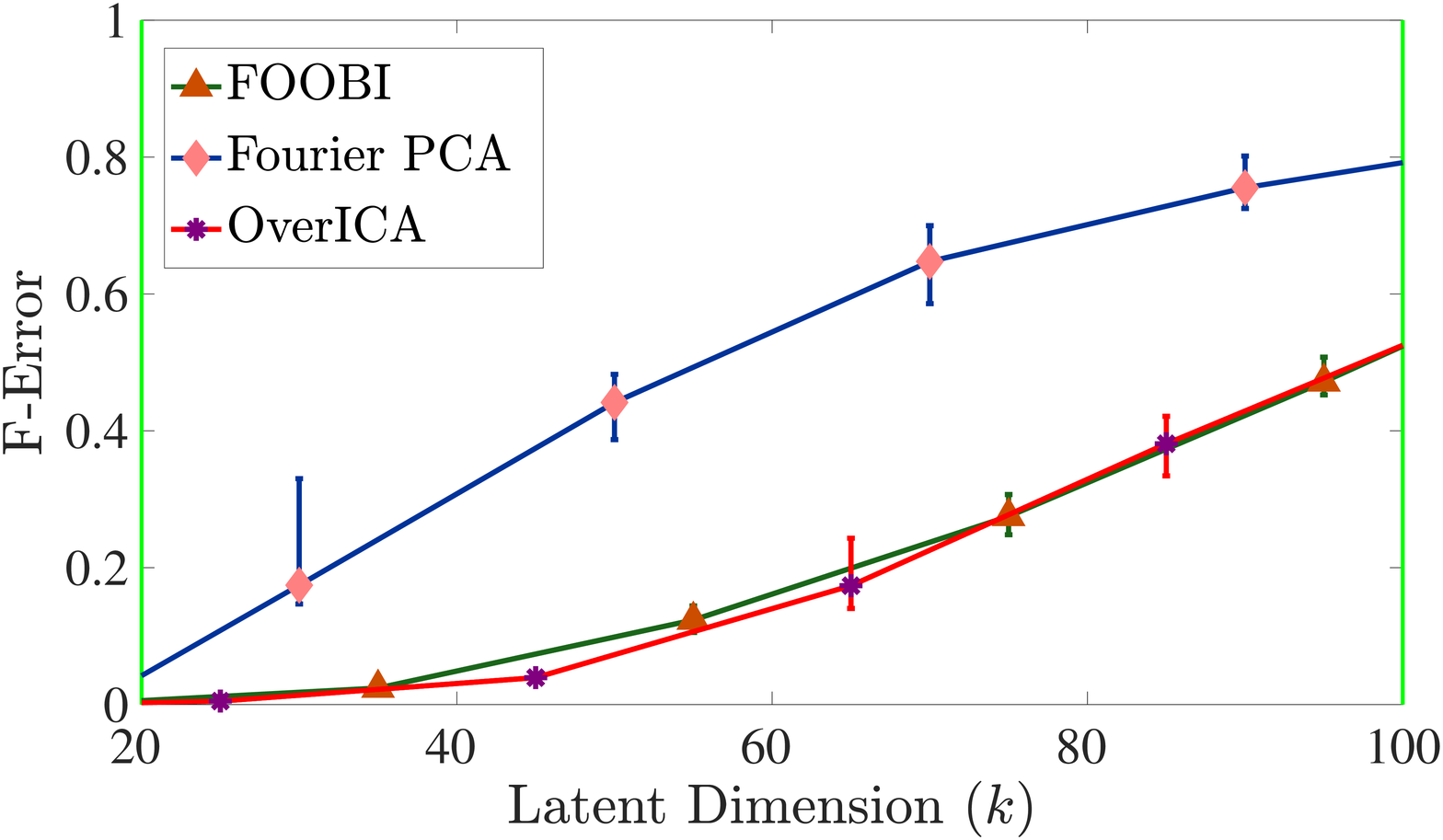}
		&
		\includegraphics[width=.33\textwidth]{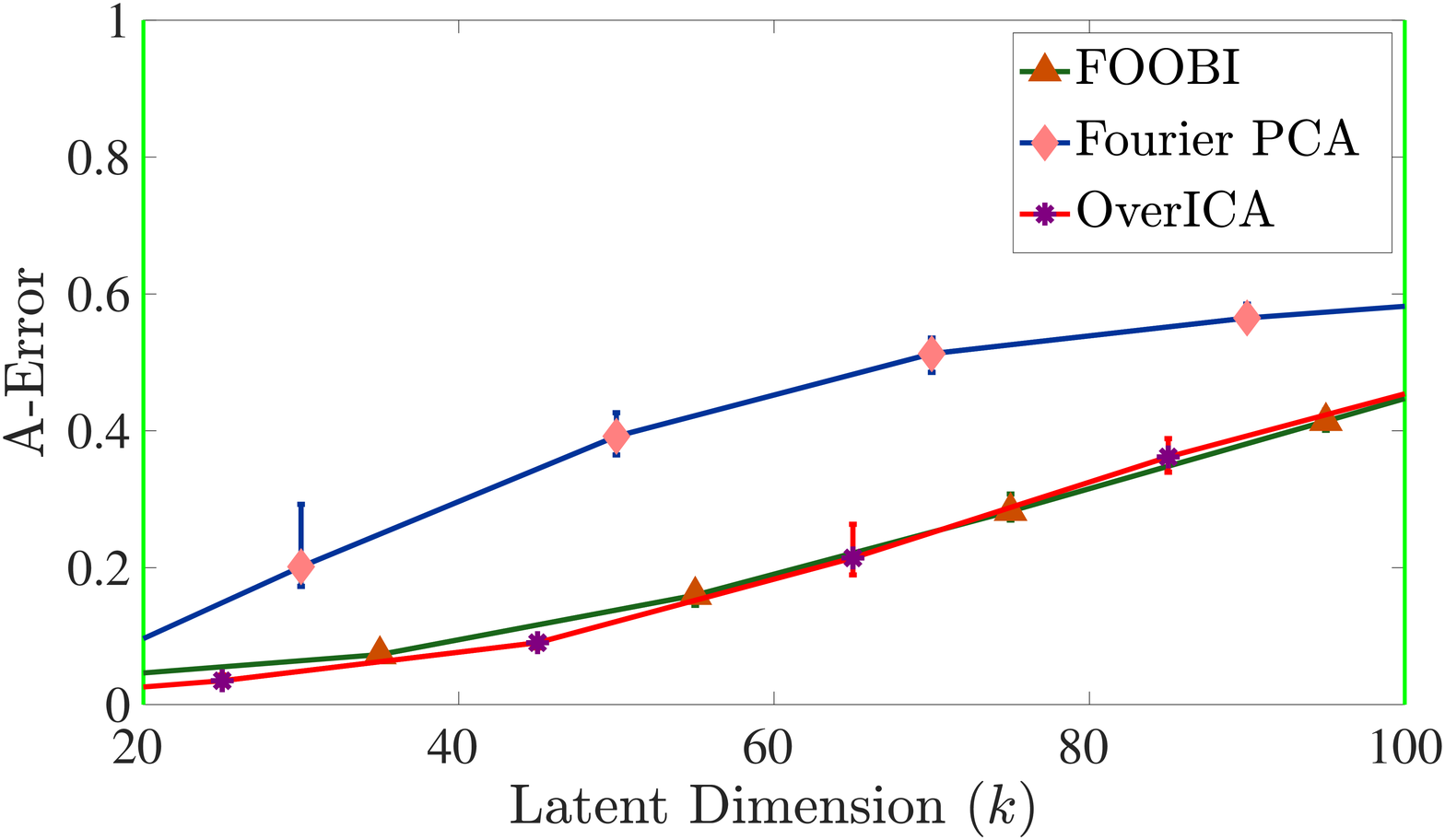}
		&
		\includegraphics[width=.33\textwidth]{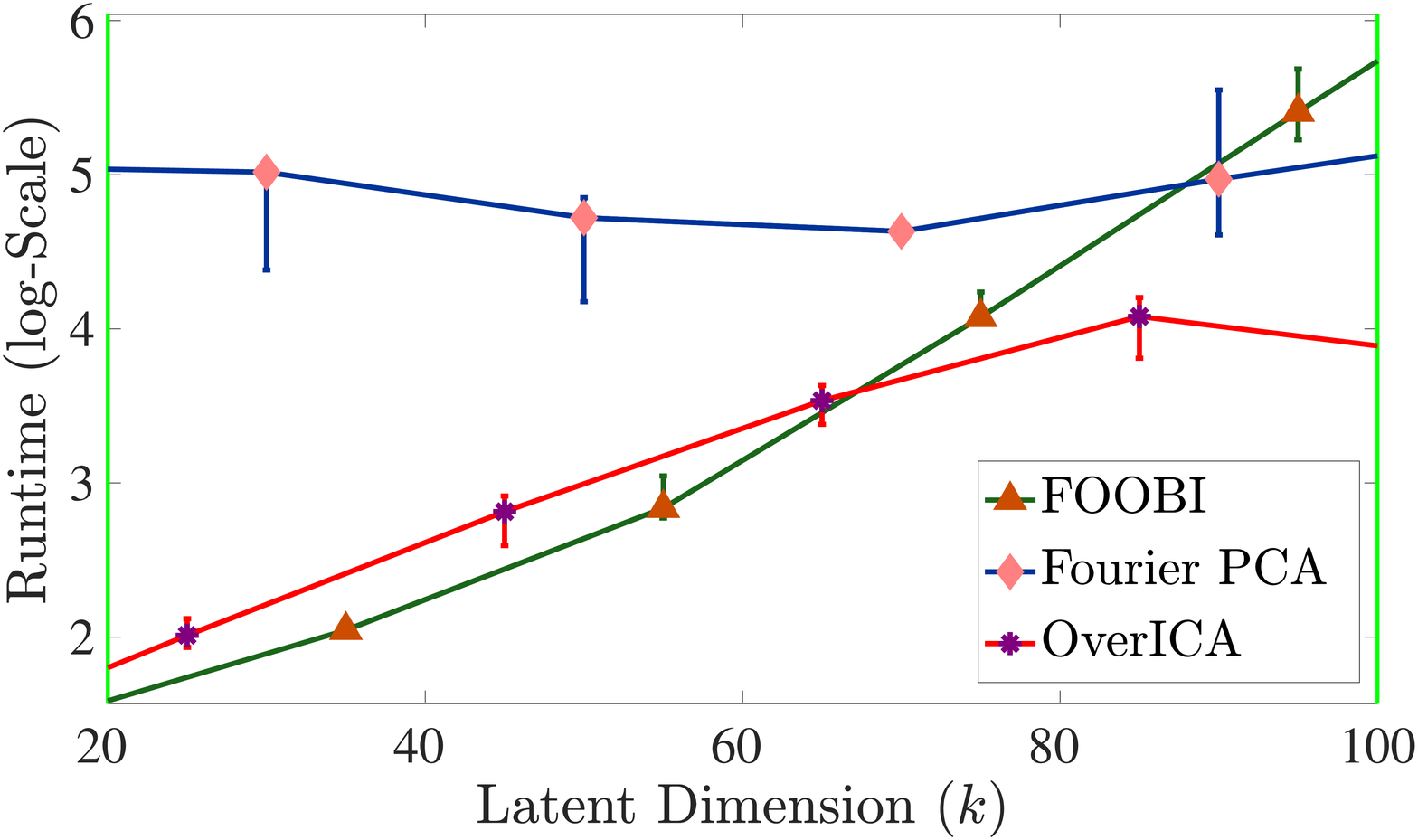}
	\end{tabular}
	\endgroup
	\caption{
		Additional plots for the runtime comparison 
		experiment from Section~\ref{sec-runtime}.
	}
	\label{fig-fs-fixedn}
\end{figure*}	
	
\subsection{Computational and Memory Complexities}
\label{app-compl}
	
	In Table~\ref{tab-compl}, we summarize the computational
	and memory complexities of the FOOBI algorithm \citep{LatEtAl2007},
	the Fourier PCA algorithm \citep{GoyEtAl2014},
	our OverICA algorithm and its modification which replaces
	the first step 
	and two different implementations,
	GenCov from Section~\ref{sec-subspace} and CUM from 
	Appendix~\ref{app-ica-cum4}, of
	the first step of our algorithm.
	
	\paragraph{Complexity of Generalized Covariances.}
	Constructing $s$ generalized covariances,
	where usually $s=O(k)$, requires $O(p^2s)$ memory
	and $O(nsp^2)$ time complexities.
	Extracting further $k$ largest singular vectors
	would require additional $O(k^2p^2)$ time,
	but the other term is dominant since $s=O(k)$
	and $n$ is larger than $k$.
	
	\paragraph{Complexity of the Fourth-Order Cumulant.}
	A flattening $C$
	of the fourth-order cumulant 
	(as described in Appendix~\ref{app-ica-cum4})
	would require
	$O(p^4)$ memory space and it can be constructed
	in $O(np^4)$ time. The algorithms further compute 
	it's $k$ largest singular vectors, which requires
	$O(k^2p^2)$ time.
	
	\paragraph{Complexity of FOOBI.}
	The first step of FOOBI is based on the construction
	of the flattening of the fourth-order cumulant
	and therefore requires the complexities presented above.
	The second step is more involved and requires 
	construction of $O(k^4)$  and $O(p^4k^2)$
	matrices and computation of the eigen decomposition
	of a $O(k^4)$ matrix. This leads to additional
	$O(p^4k^2 + k^4)$ memory and at least $O(k^6)$
	computational complexity requirements. 
	It further solves orthogonal joint matrix diagonalization
	\citep{BunEtA1993,CarSou1993,CarSou1996} which requires 
	at least $O(k^4)$ runtime per sweep.
	
	\paragraph{Complexity of Fourier PCA.}
	The complexity of Fourier PCA is dominated by the 
	first step where two fourth-order generalized cumulants
	are constructed. This requires $O(p^4)$ memory and 
	$O(np^4)$ time complexities, although we notice in practice
	that the constant hidden in $O(\cdot)$ for the time is 
	rather large.
	
	\paragraph{Complexity of OverICA.}
	Since one iteration of FISTA
	(see Algorithm~\ref{alg-fista})
	requires $O(p^3)$ and the number of iterations
	is not high, the algorithm is dominated by the first step.
	Then it takes the respective time of the construction 
	of generalized covariance or the fourth-order cumulant 
	and then computation of the SVD.

	We design the following synthetic experiment to compare the runtime.
	We sample finite sample synthetic data as described in Appendix~\ref{app-sampling-fsfixedn}
	with the fixed sample size $n=100,000$
	and observed dimension $p=20$. The latent
	dimension takes values between $k=p=20$ and $k=p^2/4=100$
	in steps of $20$. We measure runtimes in seconds and display 
	the results in log-linear scale. This comparison is for illustrative purposes only since the runtime depends on different factors. In particular,
	our Matlab/C++ code for FOOBI is highly optimized for runtime performance, 
	while our Matlab implementations of OverICA and Fourier PCA 
	are less so. The parameter $s$ for OverICA is set to $s=10k$.
	We show a head-to-head comparison of runtime in Figure~\ref{fig-fs-fixedn}.
	We observe that f- and a-errors of OverICA and FOOBI are nearly
	the same
	which is in accord with the experimental results 
	from Section~\ref{sec-exps-fs}.

	\subsection{Additional Experiments}
	\label{app-exps-supl}
	
	In this section, we present some more plots 
	for the finite sample experiment
	from Section~\ref{sec-exps-fs}
	(see Figure~\ref{fig-fs-more})
	and for the runtime experiment 
	from Section~\ref{sec-runtime}
	(see Figure~\ref{fig-fs-fixedn}).

\end{appendices}

\end{document}